\documentclass[10pt,a4paper]{article}
\usepackage[T1]{fontenc}
\usepackage{lmodern}
\usepackage[margin=2cm]{geometry}
\usepackage{amsmath,amssymb,amsthm,bm,mathtools}
\usepackage{graphicx,booktabs,xcolor,overpic,subcaption}
\usepackage{algorithm,algorithmic}
\usepackage[numbers,sort&compress]{natbib}
\usepackage[hidelinks]{hyperref}
\numberwithin{equation}{section}
\newtheorem{proposition}{Proposition}

\date{}
\title{A transition-density-based operator learning method for Fokker-Planck equations with various initial distributions}
\author{Li Zeng\thanks{School of Mathematics and Statistics, Fuzhou University, Fuzhou 350108, China. Email: \texttt{lzeng@fzu.edu.cn}.}
\and Xiaoliang Wan\thanks{Department of Mathematics and Center for Computation and Technology, Louisiana State University, Baton Rouge 70803, USA. Email: \texttt{xlwan@lsu.edu}.}
\and Yaobin Wang\thanks{Faculty of Science and Technology, Beijing Normal-Hong Kong Baptist University, Zhuhai 519087, China. Email: \texttt{wangyaobin@bnbu.edu.cn}.}
\and Fabio Nobile\thanks{Institute of Mathematics, \'Ecole Polytechnique F\'ed\'erale de Lausanne, Lausanne CH-1015, Switzerland. Email: \texttt{fabio.nobile@epfl.ch}.}
\and Tao Zhou\thanks{Institute of Computational Mathematics and Scientific/Engineering Computing, Academy of Mathematics and Systems Science, Chinese Academy of Sciences, Beijing 100190, China. Email: \texttt{tzhou@lsec.cc.ac.cn}.}}
\begin{document}
\maketitle
\begin{abstract}
We propose a transition-density-based operator learning method for solving Fokker-Planck equations (FPEs) with varying initial distributions.
A single conditional normalizing flow learns the transition probability density function (PDF) over time and initial states, preserving nonnegativity and normalization by construction. 
The learned transition PDF enables solution reconstruction through the Chapman-Kolmogorov representation for new initial distributions without retraining. To capture the singular short-time limit, we use the Gaussian transition law of a linearized stochastic differential equation around the initial state, as the base distribution and enforce an identity flow map at the initial time. Under certain assumptions, we establish short-time bounds on the discrepancy between the linearized and original transition laws. 
Combined with global bounds on the squared Fokker-Planck residual, these motivate the use of a time-weighted loss for global space-time training with adaptive sampling. For online reconstruction, we develop an importance-sampling
estimator that mixes a locally informed Gaussian proposal with
the prescribed initial distribution. Numerical experiments demonstrate accurate reconstruction across initial distributions. In a repeated-query benchmark, the proposed method achieves lower errors and lower total cost than a separately trained temporal normalizing-flow baseline.
\end{abstract}

\section{Introduction}

The Fokker-Planck equation (FPE) describes the time evolution of the probability density function (PDF) associated with a stochastic dynamical system of diffusion type, serving as a deterministic representation of the stochastic process. 
Numerically solving the FPE is of high interest in many fields, such as statistical physics, chemistry, biology, finance, and data science, as it enables uncertainty quantification and statistical inference. Moreover, in many of these applications, such as ensemble forecasting \cite{leutbecher2008ensemble} and data assimilation \cite{miller1999data}, the underlying stochastic dynamics are fixed, whereas the initial distribution varies across scenarios or posterior updates. Solving the FPE repeatedly for each initial condition can therefore lead to substantial computational costs.

Classical numerical discretization schemes for the FPE include the finite element method \cite{deng2009finite}, the finite difference method \cite{kumar2006solution}, the path integral method \cite{wehner1983numerical}, and the Jordan-Kinderlehrer-Otto (JKO) scheme \cite{jordan1998variational}, among others. Nevertheless, these traditional approaches incur high computational cost and suffer from poor scalability in high-dimensional settings.
To overcome these challenges, deep learning methods \cite{ew2018deep, raissi2019physics, sirignano2018dgm} bring a new paradigm. In general, deep learning-based methods for solving the FPE can be categorized into two main classes: those that directly represent the solution with neural networks \cite{chen2021solving,  noh2025fpl,  wang2025tensor, wen2024coupling, xu2020solving, zhai2022deep}, and those that employ generative modeling techniques, such as normalizing flows \cite{elbekri2025flowkac, feng2022solving,  lu2022learning, TANG2022111080, zeng2023adaptive, zeng2023bounded} and score-based approaches \cite{boffi2023probability, hu2025score, zhou2025simulating}. Alternatively, Neural Galerkin methods \cite{bruna2024neural, wen2024coupling} approximate the solution of the interacting-particle FPE using a Gaussian mixture model, and numerically solve the ordinary differential equation (ODE) system governing the evolution of the mixture parameters. The resulting approximation is normalized at each time step to enforce the PDF constraints.
In parallel, several techniques have been developed for some special FPEs that exploit the reinterpretation as the \(L^2\)-Wasserstein gradient flow of the Kullback-Leibler divergence \cite{lee2024deep, liu2022neural}. 

While a variety of numerical methods have shown promising performance in solving FPEs with a prescribed initial distribution, comparatively fewer methods are designed to efficiently handle varying initial distributions. 
The aforementioned methods require solving the FPE anew every time a new initial distribution arrives. This challenge can be viewed from the perspective of operator learning, where the objective is to learn a mapping from an initial distribution to the corresponding time-dependent solution of the FPE. 
Several recent works have begun to address related problems for stochastic dynamics. A Gaussian-mixture latent surrogate \cite{wang2026deep} directly learns the evolution from an initial PDF to its transient PDF representation, with a focus on Gaussian-mixture initial distributions and truncated computational domains. A pseudoreversible normalizing flow \cite{yang2024pseudoreversible} learns the transition density at a prescribed terminal time from SDE-generated samples, enabling efficient sampling and observable estimation under varying initial distributions without retraining. 
In the broader data-driven setting, generative models widely used to learn conditional distributions include conditional Glow \cite{lu2020structured}, monotone generative adversarial networks \cite{baptista2024conditional}, and conditional optimal-transport approaches based on the gradient of a partially input-convex neural network or the flow map of a Neural ODE \cite{wang2025efficient}. 

More generally, a variety of operator-learning methods have been developed to approximate PDE solution operators. Representative neural-operator architectures include the deep operator network (DeepONet) \cite{lu2021learning}, the Fourier neural operator (FNO) \cite{lifourier}, and their variants, which directly approximate mappings between input and output functions. 
Beyond these architectures, a neural-control approach \cite{gaby2024neural} approximates solution operators by learning a vector field in the neural-network parameter space and evolving the parameters through the resulting ODE system.
Another line of work constructs solution operators by learning Green's functions appearing in integral representations of linear PDE solutions \cite{aldirany2024operator,  boulle2022learning, teng2022learning}. Such kernel-based formulations avoid treating the entire input function as a direct input to the learned model; instead, the dependence on the input function is mediated through a kernel parameterized by finite-dimensional source variables and recovered through the known integral representation. Depending on the underlying PDE, the singular behavior of the Green's function may require problem-specific numerical treatment.
Among these formulations, the kernel-based viewpoint is particularly natural for Fokker-Planck equations, whose solution operator can be represented through the transition PDF of the underlying stochastic process. We adopt this viewpoint in the present work. Unlike a generic PDE kernel, the transition PDF also carries additional probabilistic structure: it is nonnegative and normalized for every initial state, while approaching a Dirac measure as \(t\to0\). These properties pose specific representation and optimization challenges and therefore call for a problem-dependent treatment of the FPE transition kernel. In this work, we propose a conditional normalizing flow with a carefully constructed base distribution, to approximate the time-dependent transition PDF over an entire time interval, and recover the solution of the FPE for a given initial distribution by a suitable importance sampling procedure.

The concurrent works \cite{fabioTPDF,saporiti2026neural} are
closely related to the present work, as they also employ conditional normalizing flows to approximate transition PDFs. Their primary emphases and
numerical formulations, however, differ from ours. In \cite{fabioTPDF},
the primary goal is to construct a structure-preserving surrogate for the transition density within a Neural Galerkin framework. The framework is further extended in \cite{saporiti2026neural} to parameterized diffusion processes with inaccessible boundaries, with particular emphasis on likelihood approximation and Bayesian inference, rather than operator learning.

In this work, we develop a global space-time formulation to learn the transition density of a fixed stochastic dynamical system. In contrast to the time-marching Neural Galerkin formulation in \cite{fabioTPDF}, the network parameters are optimized jointly over the time interval rather than evolved through a parameter-space ODE. The method combines a conditional density representation, a dynamics-informed base distribution, residual-based training, and online reconstruction, as described below.

First, we parameterize the transition PDF \(p(\bm{x},t|\bm{x}_0)\) using a conditional normalizing flow conditioned on time and the initial state. The resulting PDF is nonnegative and normalized by construction. Through the Chapman-Kolmogorov equation, the learned transition PDF can be reused to obtain approximations of the FPE solution for varying initial distributions without retraining.

Second, we use the transition law of a locally linearized SDE as the base distribution of the conditional normalizing flow to accommodate the singular short-time limit of the original transition law. 
As \(t\to 0^+\), the transition law converges weakly to the Dirac measure at the initial state, and its PDF becomes increasingly concentrated and difficult to approximate. Unlike the one-step Euler-Maruyama construction in \cite{fabioTPDF}, which freezes both drift and diffusion at \(\bm{x}_0\), we retain the first-order variation of the drift while freezing the diffusion, and use the exact transition law of the resulting linear SDE. 
Combined with an identity flow map at \(t=0\), this base distribution enforces the Dirac initial condition without an additional penalty. 
This design is also consistent with causality considerations in physics-informed learning that emphasize respecting the information prescribed at the initial time \cite{daw2023mitigating}. Under constant nondegenerate diffusion and suitable regularity assumptions on the drift, we further establish short-time bounds of \(\mathcal{O}(t)\) in total variation distance and \(\mathcal{O}(t^2)\) in Kullback-Leibler divergence between the linearized and original transition laws.
This short-time closeness suggests that a near-identity transformation may suffice to map the target transition law to the linearized base distribution at small times.

Third, we develop a time-weighted residual loss for training the conditional normalizing flow over the joint space-time-initial-state domain. Although the base distribution captures the short-time behavior of the transition law, approximation errors can still produce large PDE residuals near \(t=0\) and dominate the training objective.
We estimate dimension-dependent short-time bounds for the integrated squared residual of Gaussian approximations under different sampling measures. These bounds motivate weights of the form \(t^{d+2}\) or \(t^{\frac{d}{2}+2}\), designed to reduce the dominance of short-time residuals. An adaptive sampling strategy further uses the base distribution and the evolving PDF approximation to place collocation points in relevant regions of the space-time domain.

Fourth, we develop an importance-sampling procedure for reconstructing solutions associated with new initial distributions through the Chapman-Kolmogorov integral. At short times, the transition kernel is sharply concentrated as a function of the initial state, so direct Monte Carlo sampling from the initial distribution may yield a high-variance estimator. To address this issue, we construct an \((\bm{x},t)\)-dependent Gaussian proposal based on the locally linearized dynamics. We then mix the Gaussian proposal with the prescribed initial distribution using a time-dependent weight, emphasizing the local concentration at short times and increasing the contribution of the initial distribution at later times. Numerical experiments suggest that the mixture proposal provides a balance between short- and longer-time reconstruction accuracy. For the tested initial distributions and training settings,
the operator method also achieves lower mean relative \(L^2\)
errors than the separately trained temporal normalizing-flow
baseline, while reducing the total computational cost,
including offline training.

The remainder of this paper is organized as follows. Section \ref{sec:problem_setup} formalizes the operator learning problem for the FPE for different initial distributions. 
Section~\ref{sec:conditional_normalizing_flow} introduces the conditional normalizing-flow approximation of the transition PDF with a dynamics-informed base distribution.
Section~\ref{sec:deep_adaptive_algorithm} develops the offline training strategy, including the short-time residual analysis, the resulting time-weighted loss, and adaptive sampling.
Section~\ref{sec:IS} develops the online reconstruction of FPE solutions for new initial distributions through importance sampling of the integral in the Chapman-Kolmogorov equation.
Section~\ref{sec:experiments} presents numerical experiments validating the accuracy, robustness, and generalization of the proposed method, and Section~\ref{sec:conclusion} concludes the paper.
Additional proofs and technical derivations are provided in the Appendix.

\section{Problem setup and reformulation}\label{sec:problem_setup}

Consider the state variable $\bm{X}_t$ modeled by the following non-degenerate SDE,
\begin{equation}
\mathrm{d}\bm {X} _{t}={\boldsymbol {f }}(\bm {X} _{t})\,\mathrm{d}t+{\boldsymbol {g }}(\bm {X} _{t})\,\mathrm{d}\bm {W} _{t}, \quad t>0,
\label{sde}
\end{equation}
where $\bm{f}(\bm {X}_{t}):\mathbb{R}^d\to \mathbb{R}^d$ and $\bm{g}(\bm{X}_t):\mathbb{R}^d\to\mathbb{R}^{d\times m}, m\geq d$ denote the drift and diffusion terms, respectively, and $\bm{W}_t$ is an $m$-dimensional standard Wiener process. The PDF $p(\bm{x} ,t)$ for $\bm{X}_{t}$ satisfies the corresponding time-dependent FPE:
\begin{equation}
{\frac {\partial p}{\partial t}}=\mathcal{L}^*_{\bm{f},\bm{g}}p\coloneqq -\nabla\cdot(p\bm{f})+\frac{1}{2}\nabla\cdot\nabla\cdot(\bm{g}\bm{g}^{\rm T}p).
\label{eqn:fp}
\end{equation}
In general, equation \eqref{eqn:fp} is defined on $\mathbb{R}^d$ with the following boundary condition
\begin{equation}\label{BCs}
p(\bm{x})\to 0 \quad \mbox{as} \quad | \bm{x}|_2 \to \infty.
\end{equation}
Meanwhile, as a PDF, \(p(\bm{x},t)\) must remain non-negative and integrate to one, i.e.,
\begin{equation}\label{density_constraint}
\int_{\mathbb{R}^d}p(\bm{x},t)\mathrm{d}\bm{x}\equiv 1, \quad \mbox{and} \quad p(\bm{x},t)\geq 0, \quad \forall\, t\geq 0.
\end{equation}

In this work, we address scenarios with varying initial conditions and develop a transition-density-based operator learning method that generalizes across them, thereby eliminating the need to solve the corresponding FPE repeatedly for different initial conditions. Specifically, we consider
\begin{equation}
    \label{eqn:tFK}
    \left\{
    \begin{aligned}
&\frac{\partial p(\bm{x},t)}{\partial t}=\mathcal{L}^*_{\bm{f},\bm{g}}p(\bm{x},t),&\bm{x}\in\mathbb{R}^d, t>0,\\
&p(\bm{x},0)=p_0(\bm{x}), \quad p_0\in\mathcal{G},&\bm{x}\in\mathbb{R}^d.
\end{aligned}
\right.
\end{equation}
Here, \(\mathcal{G}\) denotes the class of probability densities on \(\mathbb{R}^d\) that vanish almost everywhere outside \(\Omega_0\). For simplicity, we take \(\Omega_0=(-1,1)^d\).

To handle varying initial densities, we introduce the transition PDF \(p(\bm{x},t|\bm{x}_0)\) conditioned on the initial state \(\bm{x}_0\). The corresponding transition PDF satisfies the same governing equation as the original FPE \eqref{eqn:tFK}, but with a Dirac Delta function as the initial condition, 
\begin{equation}
  \label{eqn:tSDE}
  \left\{
  \begin{aligned}
    &\frac{\partial p(\bm{x},t|\bm{x}_0)}{\partial t}=\mathcal{L}^*_{\bm{f},\bm{g}}p(\bm{x},t|\bm{x}_0), &\bm{x}\in\mathbb{R}^d, t>0,\\
    &p(\bm{x},0|\bm{x}_0)=\delta(\bm{x}-\bm{x}_0), &\bm{x}\in\mathbb{R}^d.
  \end{aligned}
  \right.
  \end{equation}
Applying the Chapman-Kolmogorov equation yields the solution to the FPE \eqref{eqn:tFK},
\begin{equation}
p(\bm{x},t)=\int_{\Omega_0}p(\bm{x},t|\bm{x}_0)p_0(\bm{x}_0)\mathrm{d}\bm{x}_0.
\label{eq:p_int}
\end{equation}

The following estimate relates the transition PDF approximation error to the error in the reconstructed solution for a new initial density.

\begin{proposition}
  \label{prop:p_L1}
  Let \(\Omega_0\subset\mathbb{R}^d\) be measurable, and let \(p_0\) be a PDF supported in \(\Omega_0\). For a fixed \(t>0\), assume that \(p(\cdot,t|\bm{x}_0)\) and \(\hat{p}(\cdot,t|\bm{x}_0)\) are PDFs in \(\bm{x}\) for \(\bm{x}_0\in\Omega_0\) and are jointly measurable in \((\bm{x}, \bm{x}_0)\).
Define \[p(\bm{x},t)=\int_{\Omega_0}p(\bm{x},t|\bm{x}_0)p_0(\bm{x}_0)\mathrm{d}\bm{x}_0, \quad \hat{p}(\bm{x},t)=\int_{\Omega_0}\hat{p}(\bm{x},t|\bm{x}_0)p_0(\bm{x}_0)\mathrm{d}\bm{x}_0.\] 
Then 
\begin{equation}
\|p(\bm{x},t)-\hat{p}(\bm{x},t)\|_{L^1(\mathbb{R}^d)}\leq \text{ess}\sup_{\bm{x}_0\in\Omega_0}\|p(\bm{x},t|\bm{x}_0)-\hat{p}(\bm{x},t|\bm{x}_0)\|_{L^1(\mathbb{R}^d)}.
\end{equation}
Proof. See Appendix \ref{proof_prop_p_L1}.
\end{proposition}

Thus, a uniform \(L^1\) bound on the transition PDF error controls the reconstruction error for every initial PDF supported in \(\Omega_0\).
This motivates us to solve the FPE for different initial conditions by approximating the transition PDF \(p(\bm{x}|\bm{x}_0,t)\). 

\subsection{Notation}
For convenience, we introduce the notation used throughout the paper.
For a vector \(x \in \mathbb{R}^d\), \(|x|\) denotes its Euclidean norm. 
For a matrix \(A \in \mathbb{R}^{d \times d}\), \(\|A\|\) denotes its spectral norm induced by the Euclidean norm.
Likewise, for a bilinear map \(B : \mathbb{R}^d \times \mathbb{R}^d \to \mathbb{R}^d\), \(\|B\|\) denotes its operator norm
\(
\|B\|\coloneqq \sup\limits_{x,y \in \mathbb{R}^d,\; x \neq 0,\; y \neq 0}
\frac{|B(x,y)|}{|x|\,|y|}.
\)

Given two probability measures \(\nu\) and \(\mu\) on \(\mathbb{R}^d\) with \(\nu \ll \mu\), \(H(\nu \mid \mu)\) denotes 
the relative entropy of \(\nu\) w.r.t. \(\mu\), and \(\|\nu - \mu\|_{\mathrm{TV}}\) denotes 
the total variation distance between $\nu$ and $\mu$.
\[
    H(\nu \mid \mu)\coloneqq 
 \int_{\mathbb{R}^d} \log \left(\frac{\mathrm{d}\nu}{\mathrm{d}\mu}\right)\, \mathrm{d}\nu,\quad
\|\nu - \mu\|_{\mathrm{TV}}
\coloneqq \sup_{A \in \mathcal{B}(\mathbb{R}^d)} |\nu(A)-\mu(A)|.
\]  
Furthermore, if \(\mu\ll\rho\), and \(\mathrm{d}\nu=p_{\nu}\mathrm{d}\rho\), \(\mathrm{d}\mu=p_{\mu}\mathrm{d}\rho\), then
\[
    H(\nu \mid \mu)= \int_{\mathbb{R}^d} p_{\nu}(\bm{x})\log \left(\frac{p_{\nu}(\bm{x})}{p_\mu(\bm{x})}\right)\, \mathrm{d}\rho(\bm{x}), \quad \|\nu - \mu\|_{\mathrm{TV}}=\frac{1}{2}\int_{\mathbb{R}^d}|p_{\nu}(\bm{x})-p_{\mu}(\bm{x})|\mathrm{d}\rho(\bm{x}).
\]
Moreover, given a measurable map \(T : \mathbb{R}^d \to \mathbb{R}^d\), \(T_\# \nu\) denotes the pushforward measure of \(\nu\) through \(T\), i.e., 
\(
(T_\# \nu)(A) := \nu\bigl(T^{-1}(A)\bigr)=\nu\left(\{\bm{x}\in\mathbb{R}^d: T(\bm{x})\in A\}\right)\) for any
\(A \in \mathcal{B}(\mathbb{R}^d)\). By \(\mathcal{N}(\cdot;\bm{\mu},\bm{\Sigma})\) we denote the PDF of a multivariate normal distribution with mean \(\bm{\mu}\) and covariance matrix \(\bm{\Sigma}\). For \(\bm{f}\in L^2(\mathbb{R}^d, \nu;\mathbb{R}^m)\), \(\|\bm{f}\|_{L^2(\mathbb{R}^d,\nu;\mathbb{R}^m)}\) denotes its \(L^2\)-norm w.r.t. the measure \(\nu\), i.e., \(\|\bm{f}\|_{L^2(\mathbb{R}^d,\nu;\mathbb{R}^m)}\) \(=\left(\int_{\mathbb{R}^d}|\bm{f}(\bm{x})|^2\mathrm{d}\nu(\bm{x})\right)^{1/2}.\)

\section{The conditional normalizing flow}\label{sec:conditional_normalizing_flow}

 In this section, we introduce a conditional normalizing flow to approximate the transition PDF $p(\bm{x},t|\bm{x}_0)$ given the initial state $\bm{x}_0$. Before proceeding to the conditional normalizing flow, we briefly review the idea of normalizing flows.

 A normalizing flow seeks an invertible transformation \(T\) from the unknown random vector \(\bm{X}\in\mathbb{R}^d\) to a given random vector \(\bm{Z}\in\mathbb{R}^d\) with a known distribution from which samples can be drawn efficiently. According to the change of variables formula,
the PDF of $\bm{X}=T^{-1}(\bm{Z})$ is given by
\begin{equation}
p_{\bm{X}}(\bm{x})=p_{\bm{Z}}(T(\bm{x}))\Big|\det \nabla_{\bm{x}} T(\bm{x})\Big|.
\label{eqn:variable_formula}
\end{equation}

Furthermore, by allowing the distribution of the given random variable \(\bm{Z}\) and the transformation \(T\) to depend on some parameters, the initial state \(\bm{x}_0\) and time \(t\) in our setting, we can derive a conditional normalizing flow that maps the current state \(\bm{X}_t\) of \eqref{sde} to a base random variable \(\bm{Z}_t(\bm{x}_0)\) conditioned on an initial state \(\bm{X}_0=\bm{x}_0\), i.e.,
\begin{equation}
    \label{eqn:conditional_flow}
    \begin{aligned}
    &\bm{X}_t|_{\bm{X}_0=\bm{x}_0}=T^{-1}(\bm{Z}_t(\bm{x}_0); \theta({\bm{x}_0},t)),\\
    &p_{\bm{X}_t|\bm{X}=\bm{x}_0}(\bm{x})=p_{\bm{Z}_t(\bm{x}_0)}\left(T(\bm{x}; \theta({\bm{x}_0},t))\right)\Big|\det \nabla_{\bm{x}} T(\bm{x}; \theta({\bm{x}_0},t))\Big|,
    \end{aligned}
\end{equation}
where \(\theta(\bm{x}_0,t)\) represents the parameter set of the invertible transformation \(T\) conditioned on the initial state \(\bm{x}_0\) and time \(t\). 
This transformation \(T\) maps an unknown stochastic process \(\bm{X}_t\) to another stochastic process \(\bm{Z}_t\) given the initial state \(\bm{X}_0=\bm{x}_0\).

\subsection{What makes a good base stochastic process for the conditional normalizing flow?}

The transition PDF provides a way to approximate the solution of the FPE under different initial conditions. However, 
directly approximating the solution of the FPE with a Dirac delta function as the initial condition is challenging.
The challenge lies in the fact that the Dirac delta function lacks regularity, and there is no invertible transformation that maps it to a regular distribution (e.g., normal or uniform distributions) used in normalizing flows.
To address this challenge, we linearize the original SDE and use the solution of the resulting linear SDE as the base stochastic process of the normalizing flow. 
The linearization of the original SDE is based on the assumption that the drift term \(\bm{f}\) and diffusion term \(\bm{g}\) are smooth functions. This allows us to approximate the original SDE by a linear SDE, which admits an analytical solution. Specifically, 
we approximate the original drift term by the first-order Taylor expansion and the diffusion term by the zero-order Taylor expansion at the initial state \(\bm{x}_0\), i.e., 
\begin{equation}
\label{eqn:linearized_SDE}
\left\{
\begin{aligned}
\mathrm{d}\widehat{\bm{X}}_t&=(\bm{f}(\bm{x}_0)+\nabla \bm{f}(\bm{x}_0)(\widehat{\bm{X}}_t-\bm{x}_0))\mathrm{d}t+\bm{g}(\bm{x}_0)\mathrm{d}\bm{W}_t,\qquad t>0,\\
\widehat{\bm{X}}_0&=\bm{x}_0.
\end{aligned}
\right.
\end{equation}
The corresponding solution admits an analytical form, 
\begin{equation}
\begin{aligned}
\widehat{\bm{X}}_t=\bm{x}_0+\int_0^te^{\nabla\bm{f}(\bm{x}_0)(t-s)}\bm{f}(\bm{x}_0)\mathrm{d}s+\int_0^te^{\nabla\bm{\bm{f}(\bm{x}_0)}(t-s)}\bm{g}(\bm{x}_0)\mathrm{d}\bm{W}_s,
\end{aligned}
	\end{equation}
which is a multivariate Gaussian process with mean and covariance given by
\begin{equation}
\label{eqn:mean_var_of_linearized_SDE}
\begin{aligned}
&\mathbb{E}\left[\widehat{\bm{X}}_t|\widehat{\bm{X}}_0=\bm{x}_0\right]=\bm{x}_0+\int_0^te^{\nabla\bm{f}(\bm{x}_0)(t-s)}\bm{f}(\bm{x}_0)\mathrm{d}s, \\
&\bm{\Sigma}\left[\widehat{\bm{X}}_t|\widehat{\bm{X}}_0=\bm{x}_0\right]=\int_0^te^{\nabla\bm{f}(\bm{x}_0)(t-s)}\bm{g}(\bm{x}_0)\bm{g}(\bm{x}_0)^\top e^{\nabla \bm{f}(\bm{x}_0)^\top (t-s)}\mathrm{d}s.
\end{aligned}
\end{equation}
We approximate the two integrals in equation \eqref{eqn:mean_var_of_linearized_SDE} using Gauss-Legendre quadrature: 
\begin{equation}
	\label{Gauss_integral}
	\begin{aligned}
        &\mathbb{E}\left[\widehat{\bm{X}}_t|\widehat{\bm{X}}_0=\bm{x}_0\right]\approx \bm{x}_0+\sum^m_{i=1}w_ie^{\nabla \bm{f}(\bm{x}_0)(t-s_i)}\bm{f}(\bm{x}_0),\\
	&\bm{\Sigma}\left[\widehat{\bm{X}}_t|\widehat{\bm{X}}_0=\bm{x}_0\right]\approx \sum^m_{i=1}w_ie^{\nabla \bm{f}(\bm{x}_0)(t-s_i)}\bm{g}(\bm{x}_0)\bm{g}(\bm{x}_0)^\top e^{\nabla\bm{f}(\bm{x}_0)^\top (t-s_i)},\\
	&w_i = \frac{t}{2}\hat{w}_i, \quad s_i=\frac{t}{2}(\hat{s}_i+1),
	\end{aligned}
	\end{equation}
where \(\{(\hat{w}_i, \hat{s}_i)\}\) are the weights and nodes associated with the Legendre polynomial defined on \([-1,1]\). Moreover, if \(\nabla\bm{f}(\bm{x}_0)\) is invertible, we can rewrite the solution as
\begin{equation}
\label{eqn:Solu_to_linearized_SDE}
\widehat{\bm{X}}_t=\bm{x}_0+\left(\nabla\bm{f}(\bm{x}_0)\right)^{-1}\left(e^{\nabla\bm{f}(\bm{x}_0)t}-\bm{I}\right)\bm{f}(\bm{x}_0)+\int_0^te^{\nabla\bm{\bm{f}(\bm{x}_0)}(t-s)}\bm{g}(\bm{x}_0)\mathrm{d}\bm{W}_s.
\end{equation}
The corresponding conditional mean admits the equivalent expression
\begin{equation*}
\mathbb{E}\left[\widehat{\bm{X}}_t|\widehat{\bm{X}}_0=\bm{x}_0\right]=\bm{x}_0+\left(\nabla\bm{f}(\bm{x}_0)\right)^{-1}\left(e^{\nabla\bm{f}(\bm{x}_0)t}-\bm{I}\right)\bm{f}(\bm{x}_0).
    \end{equation*}
Thus, instead of approximating the integral, one may compute the inverse of \(\nabla f(\bm{x}_0)\).

We note that the linear SDE system~\eqref{eqn:linearized_SDE} provides a close pathwise approximation to the original SDE~\eqref{sde} for small \(t\).
 The corresponding transition PDF associated with the linear SDE provides a good approximation of the target transition PDF in the short-time regime, as verified by Proposition~\ref{prop:dis_TPDFs} below.

\begin{proposition}
\label{prop:dis_TPDFs}
Suppose that the drift term \(\bm{f}=\bm{f}_{\nu}\) in the SDE \eqref{sde} is differentiable, and the diffusion term \(\bm{g}\) is constant, \(\bm{G}=(G^{ij})_{i,j\leq d}=\frac{1}{2}\bm{g}\bm{g}^\top\) is non-degenerate,
\(\left\{p_{\nu_t}\right\}_{t\in[0,t_f]}\) is the PDF solution on the time interval \([0,t_f]\) to the corresponding FPE with a Dirac delta function \(\delta_{\bm{x}_0}\) as the initial condition,
\(\left\{p_{\sigma_t}\right\}_{t\in[0,t_f]}\) is the PDF solution  to the FPE corresponding to the linearized SDE \eqref{eqn:linearized_SDE} with drift term \(\bm{f}_\sigma\) and the same initial condition \(\delta_{\bm{x}_0}\). Denote \(\nu_t\) and \(\sigma_t\) the probability measures corresponding to \(p_{\nu_t}\), \(p_{\sigma_t}\), respectively. If \(\|\nabla\bm{f}_{\nu} (\bm{x})\|\leq C_{\bm{f}}\),  \(\forall \bm{x}\in\mathbb{R}^d\) for some \(C_{\bm{f}}\geq 0\), then for any \(t\in(0,t_f)\),
\vspace{-5pt}
\begin{equation*}
	H\left(\sigma_t \mid \nu_t\right) \leq \frac{1}{2} \int_0^t \int_{\mathbb{R}^d}\left|\bm{G}^{-1 / 2} (\bm{f}_{\sigma}(\bm{x})-\bm{f}_{\nu}(\bm{x}))\right|^2 \mathrm{d} \sigma_s \mathrm{d}s,
    \vspace{-5pt}
\end{equation*}
\begin{equation*}
	\left\|\nu_t-\sigma_t\right\|_{\mathrm{TV}}^2\leq  \frac{1}{4}\int_0^t \int_{\mathbb{R}^d}\left|\bm{G}^{-1 / 2}\left(\bm{f}_{\sigma}(\bm{x})-\bm{f}_{\nu}(\bm{x})\right)\right|^2 \mathrm{d} \sigma_s \mathrm{d} s.
\end{equation*}
 
 Both the relative entropy \(H(\sigma_t|\nu_t)\) and the squared total variation distance \(\|\nu_t-\sigma_t\|^2_{\mathrm{TV}}\) are \(\mathcal{O}(t^2)\) as \(t\to 0\). 
 Moreover, if \(\|\nabla^2\bm{f}_{\nu}(\bm{x})\|\leq H_{\bm{f}}\), \( \forall \bm{x}\in\mathbb{R}^d\) for some \(H_{\bm{f}}\geq0\), then \(H(\sigma_t|\nu_t)\), \(\|\nu_t-\sigma_t\|^2_{\mathrm{TV}}\) are both \(\mathcal{O}(t^3)\) as \(t\to 0\).

 Proof. See Appendix \ref{appendix:proof_prop_dis_TPDFs}.
\end{proposition}

 Statistical quantities evaluated under the target measure can be approximated by those under the approximate measure, with an error bounded by the total variation distance between the two measures, as established in the next proposition.

\begin{proposition}
\label{prop:bound_expection_by_TV}
	Suppose that \(\nu,\sigma\) are two probability measures on \(\mathbb{R}^d\). Then for every function \(\bm{Q}\in L^2(\mathbb{R}^d, \nu;\mathbb{R}^m)\cap L^2(\mathbb{R}^d, \sigma;\mathbb{R}^m)\), we have
	\begin{equation*}
	\left|\int_{\mathbb{
			R}^d}\bm{Q}(\bm{x})\mathrm{d}\nu(\bm{x})-\int_{\mathbb{R}^d}\bm{Q}(\bm{x})\mathrm{d}\sigma(\bm{x})\right|\leq \sqrt{2\|\nu-\sigma\|_{\mathrm{TV}}}\left(\|\bm{Q}\|_{L^2(\mathbb{R}^d,\nu;\mathbb{R}^m)}+\|\bm{Q}\|_{L^2(\mathbb{R}^d,\sigma;\mathbb{R}^m)}\right).
	\end{equation*}
    Proof. See Appendix \ref{proof_prop_bound_expection_by_TV}.
\end{proposition}

In particular, if we take \(\sigma=T_{\#}\nu\) for some measurable map \(T:\mathbb{R}^d\to\mathbb{R}^d\)
and \(Q(\bm{x})=\bm{x}\), assuming \(\nu\), \(\sigma\) have finite second moments, we have
\begin{equation*}
\begin{aligned}
	\left|\int_{\mathbb{R}^d}\bm{x}\mathrm{d}\nu(\bm{x})-\int_{\mathbb{R}^d}T(\bm{x})\mathrm{d}\nu(\bm{x})\right|
    \leq \sqrt{2\|\nu-\sigma\|_{{\mathrm{TV}}}}\left(\|\bm{x}\|_{L^2(\mathbb{R}^d,\nu;\mathbb{R}^d)}+\|\bm{x}\|_{L^2(\mathbb{R}^d,\sigma;\mathbb{R}^d)}\right).
    \end{aligned}
\end{equation*}

We also state that if \(\nu_i\rightarrow \mu\) in total variation and their second moments converge, then the Knothe-Rosenblatt (KR) map \(T_i\) satisfying \((T_i)_{\#}\mu=\nu_i\) converges to the identity map in \(L_2(\mu)\).

\begin{proposition}
\label{prop:L_2_convergence_of_T}
	Let $\{\nu_i\}_{i=1}^{\infty}$ be a family of absolutely continuous probability measures on $\mathbb{R}^d$ w.r.t. the Lebesgue measure, each with finite second moments, and assume that \(\mu\) is an absolutely continuous measure, \(\|\nu_i-\mu\|_{\mathrm{TV}}\to 0\), \(\|\bm{x}\|_{L^2(\mathbb{R}^d, \nu_i)}\to \|\bm{x}\|_{L^2(\mathbb{R}^d, \mu)}\) as \(i\to\infty\). Let \(T_i\) be the KR map between $\mu$ and $\nu_i$. Then \(T_i\to \mathrm{Id}\) in \(L^2(\mu)\) as \(i\to\infty\).

      Proof. See Appendix \ref{proof_prop_L2_convergence_of_T}.
\end{proposition}

Proposition~\ref{prop:dis_TPDFs} shows that, for small \(t\), the transition PDF of the linearized SDE is close to that of the original SDE in total variation. Proposition~\ref{prop:L_2_convergence_of_T} implies that,  under certain conditions, KR maps associated with a sequence of measures and its limiting measure converge to the identity map. 
This observation motivates our design choice: the conditional normalizing flow is constrained to be the identity at \(t=0\) and to deform gradually as \(t\) increases. In this way, we avoid directly mapping from a Dirac delta initial condition and reduce the training difficulty near \(t=0\). 

\subsection{Architecture of the conditional normalizing flow}
The transition-density-based method reformulates the original problem of finding \(p(\bm{x},t|\bm{x}_0)\) as the task of identifying an invertible transformation \(T(\cdot;\theta(\bm{x}_0,t))\), mapping a stochastic process \(\bm{X}_t\) to \(\bm{Z}_t\) given the initial state \(\bm{X}_0=\bm{x}_0\). In practice, we can
construct such a complex invertible mapping by stacking a sequence of simple bijections, i.e., 
\begin{equation*}
\bm{Z}_t|_{\bm{Z}_0=\bm{Z}_0(\bm{x}_0)}=T(\bm{X}_t|_{\bm{X}_0=\bm{x}_0};\theta(\bm{x}_0,t) )=T_{[K]}\circ T_{[K-1]}\circ \cdots \circ T_{[1]}(\bm{X}_t|_{\bm{X}_0=\bm{x}_0}),
\end{equation*}
where $T_{[i]}=T_{[i]}(\cdot;\theta_{[i]}(\bm{x}_0,t))$ is based on shallow neural networks. The inverse and Jacobian determinant are given as 
\begin{align*}
&\bm{X}_t|_{\bm{X}_0=\bm{x}_0} = T^{-1}\left(\bm{Z}_t|_{\bm{Z}_0=\bm{z}_0};\theta(\bm{x}_0,t)\right)=T_{[1]}^{-1}\circ \cdots \circ T_{[K-1]}^{-1}\circ T^{-1}_{[K]}(\bm{Z}_t|_{\bm{Z}_0=\bm{Z}_0(\bm{x}_0)}),\\
&\vert \det \left(\nabla_{\bm{x}} T(\cdot;\theta(\bm{x}_0,t))\right)\vert = \prod_{i=1}^L\vert \det \left(\nabla _{\bm{x}_{[i-1]}}T_{[i]}(\cdot)\right)\vert,
\end{align*}
which motivates the use of techniques that enable efficient computation of both the inverse and the Jacobian matrix of \(T_{[i]}(\cdot)\). In this work, we adopt the structure with descending active transformation dimensions in the standard KRnet \cite{tang2020deep}, which is inspired by the KR rearrangement \cite{santambrogio2015optimal} and affine coupling layers \cite{dinh2017density}, to enhance accuracy and reduce model complexity in high-dimensional settings. 

Let \(\bm{x}^\top=\left(\left(\bm{x}^{(1)}\right)^{\top},\left(\bm{x}^{(2)}\right)^{\top}, \dots, \left(\bm{x}^{(K)}\right)^{\top}\right)\)
be a partition of \(\bm{x}^\top\), where \(\bm{x}^{(i)}=\left(x_1^{(i)},\dots,  x_{d_i}^{(i)}\right)^{\top}\) with \(1\leq K\leq d\), \(1\leq d_i\leq d\) and \(\sum_{i=1}^Kd_i=d\). Denote \(\bm{x}^{(i:j)}=\left(\left(\bm{x}^{(i)}\right)^{\top},\left(\bm{x}^{(i+1)}\right)^{\top}, \dots, \left(\bm{x}^{(j)}\right)^{\top}\right)^{\top}\). 
Our conditional normalizing flow \(T_{\text{C-KRnet}}\) is defined as follows,
\begin{align*}
& T_{[1]}=\tilde{T}_{K}, \quad T_{[k]}=\left(\begin{aligned}
    &\tilde{T}_{K+1-k}\\
    &\text{Id}_{K+2-k:K}
\end{aligned}\right), \quad k=2,\dots, K,\\
&\bm{Z}_t|_{\bm{Z}_0=\bm{z}_0}=T_{\text{C-KRnet}}(\bm{X}_t|_{\bm{X}_0=\bm{x}_0};\theta(\bm{x}_0,t) )=\left(\;\begin{aligned}
&\tilde{T}_1\\
&\text{Id}_{(2:K)}\\
\end{aligned}\;\right)\circ\cdots\circ \left(\;
\begin{aligned}
&\tilde{T}_{K-1}\\
&\text{Id}_{(K)}\\
\end{aligned}\;\right)\circ\tilde{T}_K(\bm{X}_t|_{\bm{X}_0=\bm{x}_0}),
\end{align*}
where \(\text{Id}_{(i:j)}(\bm{x})=\bm{x}^{(i:j)}\) denotes the projection operator, \(\tilde{T}_k:[-1,1]^{\sum^k_{i=1}d_i}\to[-1,1]^{\sum^k_{i=1}d_i}\) is an invertible mapping of \(\bm{x}^{(1:k)}\) that is constructed by the composition of conditional affine coupling layers defined in Section~\ref{sec:conditional_affine_coupling_layer}. 

The conditional normalizing flow \(T_{\text{C-KRnet}}\) is primarily constructed using two nested loops: an outer loop and an inner loop. The outer loop consists of $K$ stages, each corresponding to a mapping $\tilde{T}_k$. The inner loop $\tilde{T}_k$ contains $l_{k}$ stages, representing the number of coupling layers. Following the application of \(\tilde{T}_k\), the $k$th partition is kept fixed in the subsequent outer stages. As the outer loop progresses, the active dimension gradually decreases, which motivates reducing the depth \(l_k\) for later stages.

\subsubsection{Conditional affine coupling layer}
\label{sec:conditional_affine_coupling_layer}
In this section, we focus on the construction of the inner loop \(\tilde{T}_k=\tilde{T}_{k,l_k}\circ\cdots\circ\tilde{T}_{k,2}\circ\tilde{T}_{k,1}\). To this end, we first introduce a conditional affine coupling layer, given some information \(\bm{\xi}\) and time \(t\), applied to a vector \(\bm{y}=(\bm{y}_1, \bm{y}_2)\in\mathbb{R}^m\), where \(\bm{y}_1\in\mathbb{R}^{m_1}\), and \(\bm{y}_2\in\mathbb{R}^{m-m_1}\) as follows:
\begin{equation}
\label{affine_layer_t}
\left\{\begin{aligned}
& \hat{\bm{y}}_{1}=\bm{y}_{1},\\
& \hat{\bm{y}}_{2}=\bm{y}_{2}\odot \big(\bm{1}_{m-m_1}+\beta\tanh (t\cdot\bm{s})\big) + e^{\bm{\zeta}}\odot\tanh(t\cdot\bm{q}),
\end{aligned}\right.
\end{equation}
where \(\beta\in(0,1)\) is user defined, \(\bm{\zeta}\) is trainable and $\bm{s}$ and $\bm{q}$  are outputs of a neural network that takes as input the unchanged portion \(\bm{y}_1\), the given information \(\bm{\xi}\) and time \(t\), i.e.,
\begin{equation}
(\bm{s},\bm{q}) =\mathrm{NN}(\bm{y}_{1}, \bm{\xi}, t).
\label{NN_i_t}
\end{equation}

To emphasize the influence of the initial state, we adopt the approach proposed in \cite{he2025adaptive}, applying a random Fourier feature transformation \cite{wang2021eigenvector} before the fully connected layers of the neural network used in equation~\eqref{NN_i_t}. The resulting neural network is 
\begin{equation}
\label{eqn:NN_structure}
\begin{aligned}
    &\bm{h}_0=\left[\bm{y}_{1},\bm{\xi}, t\right]^\top,\quad
    \bm{h}_1=\left[
        \sin\left(\frac{1}{e^{\gamma}}\bm{F}\bm{h}_0+\bm{b}_0\right),
        \cos\left(\frac{1}{e^{\gamma}}\bm{F}\bm{h}_0+\bm{b}_0\right),
        \quad \bm{h}_0\right]^\top,\\
    &\bm{h}_j = \text{SiLU}(\bm{W}_{j-1}\bm{h}_{j-1}+\bm{b}_{j-1}), \text{  for } j=2,\dots, M,\quad
    \left(
        \bm{s},
        \bm{q}
    \right)=\bm{W}_M\bm{h}_M+\bm{b}_M,
\end{aligned}
\end{equation}
where \(\bm{F}\in\mathbb{R}^{r_h/2\times\text{dim}(h_0)}\), \(\bm{b}_0\in\mathbb{R}^{r_h/2}\) are sampled from a standard normal distribution and a uniform distribution over \([0,2\pi]^{r_h/2}\), respectively. Here, \(r_h\) denotes the dimension of the random features. Both \(\bm{F}\) and \(\bm{b}_0\) are fixed after initialization. \(\gamma, \{\bm{W}_{j}, \bm{b}_j\}_{j=1}^{M}\) are trainable parameters, and the sigmoid linear unit (SiLU) function is used as the activation function.

Since the conditional affine coupling layer \eqref{affine_layer_t} updates only \(\bm{y}_2\), we alternate the roles of \(\bm{y}_1\) and \(\bm{y}_2\) in subsequent layers to ensure a full update of \(\bm{y}\). 
Note that the transformation defined by the conditional affine coupling layer \eqref{affine_layer_t} reduces to the identity when \(t=0\), which is consistent with using the PDF corresponding to the linearized SDE as the base distribution of the conditional normalizing flow. 

Now we are ready to construct \(\tilde{T}_k\), which consists of \(l_k\) conditional affine coupling layers.
Denote \(\bm{x}_{[0]}=\bm{x}, \bm{x}_{[k]}=T_{[k]}(\bm{x}_{[k-1]})\), \( \bm{\xi}=\bm{x}_0\), and 
\(\bm{y}_{[k]}=\bm{x}_{[k]}^{(1:K-k)}\) the active part of \(\bm{x}_{[k]}\) for \(k=1,\dots, K-1\). We apply a sequence of conditional affine coupling transformations \eqref{affine_layer_t} to \(\bm{y}_{[k]}\) using a half-half split.

\section{Deep adaptive algorithm}\label{sec:deep_adaptive_algorithm}

Given an SDE, we solve the corresponding transition FPE with Dirac delta functions centered at various points as initial conditions. This constitutes the offline training stage. 
Let \(p(\bm{x}, t|\bm{x}_0)\) be a conditional probability density function defined on \(\mathbb{R}^d\times[0,t_f]\times\Omega_0\). It satisfies the transition FPE \eqref{eqn:tSDE}.

By using the solution of the linearized SDE as the base stochastic process and enforcing the transformation to be the identity mapping at \(t=0\), the initial condition is naturally satisfied.
Applying PINNs yields
\begin{equation}
	\begin{aligned}
		&\mathcal{L}(\bm{\theta})=\int_{\mathbb{R}^d\times(0, t_f]\times\Omega_0}|r(\bm{x},\bm{x}_0, t; p_{\bm{\theta}})|^2\mathrm{d}\rho(\bm{x},\bm{x}_0, t), \quad r(\bm{x},\bm{x}_0, t;p_{\bm{\theta}})=\frac{\partial p_{\bm{\theta}}}{\partial t}-\mathcal{L}^*_{\bm{f},\bm{g}}[p_{\bm{\theta}}],
	\end{aligned}
	\label{eqn:loss}
\end{equation}
where \(\rho(\bm{x},\bm{x}_0, t)=\rho(\bm{x}|\bm{x}_0, t)\rho(t|\bm{x}_0)\rho(\bm{x}_0)\) is the sampling measure to be used in the training process. For the case where \(\bm{x}_0\) lies within a bounded domain \(\Omega_0\) and \(t_f\) is finite, we define \(\rho(t|\bm{x}_0)\rho(\bm{x}_0)\) to be the uniform measure over \(\Omega_0\times [0, t_f]\). For the remaining term, \(\rho(\bm{x}|\bm{x}_0, t)\), we employ an adaptive sampling strategy similar to that used for specific FPEs in our previous work \cite{feng2022solving, zeng2023adaptive}. The only difference is that, instead of sampling solely from the uniform distribution at the start, we are motivated to sample from the base stochastic process, as it approximates the target transition PDF particularly well when the time \(t\) is small. After some training epochs, we obtain a conditional normalizing flow that approximates the target solution and can be used to generate samples from it. These samples are then used to update the collocation points. The procedure offers an adaptive strategy for designing the sampling measure \(\rho_k(\bm{x}|\bm{x}_0,t)\) at stage \(k\), by combining the uniform measure, the measure from the previous stage, and the one induced by the current normalizing flow, i.e.,
\begin{equation}
\label{eqn:sampling_measure}
\begin{aligned}
    &\rho_0(\bm{x}|\bm{x}_0,t)=\frac{\gamma_1}{\gamma_1+\gamma_3}\mu(\bm{x})+\frac{\gamma_3}{\gamma_1+\gamma_3}\rho_{\bm{Z}}(\bm{x}|\bm{x}_0, t), \quad \gamma_1+\gamma_2+\gamma_3=1,\\
    &\rho_{k+1}(\bm{x}|\bm{x}_0,t) =\gamma_1\mu(\bm{x})+\gamma_2\rho_k(\bm{x}|\bm{x}_0,t)+\gamma_3 \rho_{\text{C-KRnet},\bm{\theta}}(\bm{x}|\bm{x}_0,t), 
\end{aligned}
\end{equation}
where \(\mu\) denotes the uniform measure over a user-defined bounded domain, \(\rho_{\bm{Z}}(\bm{x},t|\bm{x}_0)\) denotes the measure induced by the base random process, \(\rho_{\text{C-KRnet},\bm{\theta}}(\bm{x},t|\bm{x}_0)\) denotes the measure induced by the current conditional normalizing flow and \(\gamma_1, \gamma_2, \gamma_3\) are hyperparameters used to tune the sampling measure. 

The transition PDF of the base stochastic process serves as a good approximation of the target transition PDF when \(t\) is small. However, as \(t\to0\), the target solution approaches a Dirac delta function, and even small perturbations may lead to the divergence of the residual. To illustrate this, we consider a small perturbation around the transition PDF of the linearized SDE and analyze the asymptotic behavior of the corresponding residual as \(t\to 0\). 

\begin{proposition}
\label{prop:residual_approx}
Suppose the drift term \(\bm{f}\) in the SDE \eqref{sde} is differentiable, and the diffusion term \(\bm{g}\) is constant, \(\bm{G}=(G^{ij})_{i,j\leq d}=\frac{1}{2}\bm{g}\bm{g}^\top\) is non-degenerate. Assume that \(\|\nabla\bm{f} (\bm{x})\|\leq C_{\bm{f}}\),  \(\forall \bm{x}\in\mathbb{R}^d\) for some \(C_{\bm{f}}\geq 0\).
 Given an initial state \(\bm{x}_0\) and time \(t\), let \(\tilde{p}(\bm{x},t|\bm{x}_0)=\mathcal{N}(\bm{x};\bm{m}(t), \bm{\Sigma}(t))\) where \(\bm{\Sigma}(t)\sim t\bm{G}\), and \(\frac{\mathrm{d}}{\mathrm{d}t}\bm{m}(t)\) is bounded. Then, 
    \begin{align}
        &\int_{\mathbb{R}^d} |r(\bm{x},\bm{x}_0, t;\tilde{p})|^2\mathrm{d}\bm{x}= \mathcal{O}(t^{-(\frac{d}{2}+2)}),\label{eqn:residual_approx_uniform}\\
        &\mathbb{E}_{\bm{x}\sim\tilde{p}}\left[|r(\bm{x},\bm{x}_0, t;\tilde{p})|^2\right]=\int_{\mathbb{R}^d}|r(\bm{x},\bm{x}_0, t;\tilde{p})|^2\tilde{p}(\bm{x},t|\bm{x}_0)\mathrm{d}\bm{x}
 =\mathcal{O}(t^{-(d+2)}). \label{eqn:residual_approx_flow}
    \end{align}
    
    Proof. See Appendix~\ref{appendix:proof_prop_residual_approx}.
\end{proposition}

Proposition~\ref{prop:residual_approx} provides short-time upper bounds of order \(t^{-(\frac{d}{2}+2)}\) for the squared residual integrated with respect to Lebesgue measure and \(t^{-(d+2)}\) for its expectation under the Gaussian density \(\tilde{p}\). These bounds motivate the respective time-dependent weights \(t^{(\frac{d}{2}+2)}\) and \(t^{d+2}\), which compensate for the possible singular growth as \(t\to 0^+\). 
Incorporating the adaptive sampling strategy, we rewrite the loss function at the \(k-\)th stage as a sum of three parts:
\vspace{-3pt}
\begin{equation*}
\begin{split}
    {\mathcal{L}^k}(\bm{\theta})
    =&\,\gamma_1\int_{\Omega\times\Omega_0\times(0, t_f]}|r(\bm{x},\bm{x}_0, t;p_{\bm{\theta}})|^2\mathrm{d}\mu(\bm{x})\mathrm{d}\rho(t|\bm{x}_0)\mathrm{d}\rho(\bm{x}_0)\\
    &\, + \gamma_2\int_{\Omega\times\Omega_0\times(0, t_f]}|r(\bm{x},\bm{x}_0, t;p_{\bm{\theta}})|^2\mathrm{d}\rho_{k-1}(\bm{x}|\bm{x}_0,t)\mathrm{d}\rho(t|\bm{x}_0)\mathrm{d}\rho(\bm{x}_0)\\
    &\, + \gamma_3\int_{\Omega\times\Omega_0\times(0, t_f]}|r(\bm{x},\bm{x}_0, t;p_{\bm{\theta}})|^2\mathrm{d}\rho_{\text{C-KRnet},\bm{\theta}}(\bm{x},t|\bm{x}_0)\mathrm{d}\rho(t|\bm{x}_0)\mathrm{d}\rho(\bm{x}_0).
\end{split}
\end{equation*}

We apply weighting functions \(w_1(t)=t^{\frac{d}{2}+2}\) to the first part, and \(w_2(t)=t^{d+2}\) to the last two parts. The resulting weighted loss function is then defined as
\begin{equation*}
\begin{split}
    {\mathcal{L}}_{w}^k(\bm{\theta})
    =&\, \gamma_1\int_{\Omega\times\Omega_0\times(0, t_f]}|r(\bm{x},\bm{x}_0, t;p_{\bm{\theta}})|^2t^{(\frac{d}{2}+2)}\mathrm{d}\mu(\bm{x})\mathrm{d}\rho(t|\bm{x}_0)\mathrm{d}\rho(\bm{x}_0)\\
    & +\gamma_2\int_{\Omega\times\Omega_0\times(0, t_f]}|r(\bm{x},\bm{x}_0, t;p_{\bm{\theta}})|^2t^{(d+2)}\mathrm{d}\rho_{k-1}(\bm{x}|\bm{x}_0, t)\mathrm{d}\rho(t|\bm{x}_0)\mathrm{d}\rho(\bm{x}_0)\\
    &+\gamma_3\int_{\Omega\times\Omega_0\times(0, t_f]}|r(\bm{x},\bm{x}_0, t;p_{\bm{\theta}})|^2t^{(d+2)}\mathrm{d}\rho_{\text{C-KRnet}}(\bm{x}|\bm{x}_0, t)\mathrm{d}\rho(t|\bm{x}_0)\mathrm{d}\rho(\bm{x}_0).
\end{split}
\end{equation*}

We emphasize that the use of the linearized SDE respects causality, and the weighting functions introduced here not only help prevent blow-up but also assign greater weight to more challenging training regions, making our method well-suited for the transition FPE.

We introduce an indicator variable to distinguish samples drawn from different distributions. Specifically, let \(N=N_1+N_2+N_3\) with proportions \(N_1:N_2:N_3=\gamma_1:\gamma_2:\gamma_3\).
Define \(S_1=\left\{\left(\bm{x}^i, \bm{x}_0^i, t^i, \eta^i=0\right)\right\}_{i=1}^{N_1}\), where \(\left(\bm{x}^i, \bm{x}_0^i, t^i\right)\) are sampled from \(\mu(\bm{x})\rho(t|\bm{x}_0)\rho(\bm{x}_0)\). Similarly, define
\(S_2=\left\{\left(\bm{x}^i, \bm{x}_0^i, t^i, \eta^i=1\right)\right\}_{i=N_1+1}^{N_1+N_2}\) where \(\left\{\left(\bm{x}^i, \bm{x}_0^i, t^i\right)\right\}\) are drawn from \(\rho_{k-1}(\bm{x}|\bm{x}_0,t)\rho(t|\bm{x}_0)\rho(\bm{x}_0)\). Define \(S_3=\left\{\left(\bm{x}^i, \bm{x}_0^i, t^i, \eta^i=1\right)\right\}_{i=N_1+N_2+1}^{N}\) where \(\left\{\left(\bm{x}^i, \bm{x}_0^i, t^i \right)\right\}\) are sampled from 
\(\rho_{\text{C-KRnet}}(\bm{x}|\bm{x}_0,t)\rho(t|\bm{x}_0)\rho(\bm{x}_0)\). Let \(S=S_1\cup S_2\cup S_3\). The empirical time-weighted loss function is then defined as:
\begin{equation}
	\label{eqn:emp_loss}
	\widehat{\mathcal{L}}_{w}^k(\bm{\theta}; S) \coloneqq \frac{1}{N}\sum_{i=1}^{N}\vert r (\bm{x}^i,\bm{x}^i_0,t^i; p_{\bm{\theta}})\vert ^2\left(\eta^i t^{d+2}+(1-\eta^i)t^{\frac{d}{2}+2}\right).
\end{equation}
The optimal parameter $\bm{\theta}^*$ can be obtained by solving the optimization problem:
\begin{equation}
\bm{\theta}^*\in\arg\min_{\bm{\theta}}\widehat{\mathcal{L}}_{w}(\bm{\theta};\tilde{S}).
\label{eqn:opt_theta}
\end{equation}

The adaptive training process is summarized in Algorithm~\ref{alg:operator_learning}. The training set is divided into several mini-batches, and the Adam optimizer is used to update the parameters \(\bm{\theta}\). The training set \(S\) is updated according to the strategy \eqref{eqn:sampling_measure}.

\begin{algorithm}
	\caption{Solving the transition FPE with various initial states}
	\label{alg:operator_learning}
	\begin{algorithmic}
		 \STATE \textbf{Input:} maximum number of epochs $N_e$, maximum number of iterations $N_{\mathrm{adaptive}},$ sampling proportions $\gamma_1, \gamma_2, \gamma_3$, initial learning rate $l_r$, decay rate $\eta$, step size $n_s$, the number of the training samples \(N\).
		\FOR {$k=0,\cdots,N_{\mathrm{adaptive}}$}
            \IF {\(k=0\)}
            \STATE \(N_1=\lfloor\frac{\gamma_1}{\gamma_1+\gamma_3}*N\rfloor\), \(N_3=N-N_1\),
            \STATE Generate \(S_1=\{(\bm{x}^i, \bm{x}_0^i, t^i, 0)\}_{i=1}^{N_1}\) with \((\bm{x}^i, \bm{x}_0^i, t^i)\sim \mu(\bm{x})\rho(t|\bm{x}_0)\rho(\bm{x}_0)\),
            \STATE   \(S_3=\{(\bm{x}^i, \bm{x}_0^i, t^i, 1)\}_{i=N_1+1}^{N}\) with \((\bm{x}^i, \bm{x}_0^i, t^i)\sim \rho_{\text{C-KRnet},\bm{\theta}}(\bm{x},t|\bm{x}_0)\rho(t|\bm{x}_0)\rho(\bm{x}_0)\).
           
           Initialize the training set \(S=S_1\cup S_3\).
            \ELSE 
            \STATE \(N_1=\lfloor\gamma_1*N\rfloor\), \(N_2=\lfloor\gamma_2*N\rfloor\), \(N_3=N-N_1-N_2\),
            \STATE  Generate \(S_1=\{(\bm{x}^i, \bm{x}_0^i, t^i, 0)\}_{i=1}^{N_1}\) with \((\bm{x}^i, \bm{x}_0^i, t^i)\sim \mu(\bm{x})\rho(t|\bm{x}_0)\rho(\bm{x}_0)\),
            \STATE  \(S_2\) consists of \(N_2\) random samples drawn from \(S\).
            \STATE  \(S_3=\{(\bm{x}^i, \bm{x}_0^i, t^i, 1)\}_{i=N_1+N_2+1}^{N}\), \((\bm{x}^i, \bm{x}_0^i, t^i)\sim \rho_{\text{C-KRnet},\bm{\theta}}(\bm{x},t|\bm{x}_0)\rho(t|\bm{x}_0)\rho(\bm{x}_0)\).
            
            \STATE Update the training set: \(S=S_1\cup S_2\cup S_3\).
            \ENDIF
		\FOR {$j=1,\cdots,N_e$}
            \IF{$((k-1)N_e + j)\%n_s==0$}
		\STATE $l_r=\eta*l_r$.
		\ENDIF 
		\STATE Divide $S$ into $n$ mini-batches  $\{S^{ib}\}_{ib=1}^n$ randomly.
		\FOR{$ib = 1,\cdots,n$}
		\STATE Compute the loss function \eqref{eqn:emp_loss} $\widehat{\mathcal{L}}_{w}(\bm{\theta}; S^{ib})$.
		\STATE Update $\bm{\theta}$ by using the Adam optimizer.
		\ENDFOR
		\ENDFOR 
		\ENDFOR 
		\STATE \textbf{Output:} The predicted solution $p_{\text{C-KRnet},\bm{\theta}}(\bm{x},t|\bm{x}_0)$.
	\end{algorithmic}
\end{algorithm}

\section{Importance sampling-based approximation of \(p(\cdot,t)\)}\label{sec:IS}

Once the transition PDF is available, a natural way to approximate the solution of equation \eqref{eqn:tFK} is by approximating the integral \eqref{eq:p_int} through the Monte Carlo method, 
\begin{equation}
  p(\bm{x},t)\approx \frac{1}{M}\sum_{i=1}^M p(\bm{x},t|\bm{x}_0^i), \quad \bm{x}_0^i\sim p_0(\bm{x}),
  \label{eq:p_int_MC}
  \end{equation}
  where $\{\bm{x}_0^i\}_{i=1}^M$ are samples from the initial distribution \(p_0\).
However, the transition PDF approximates the Dirac delta function as \(t\to 0\). Consequently, a direct Monte Carlo method may incur large errors due to high variance. Alternatively, we apply importance sampling, leveraging the knowledge of the transition PDF. 
\begin{equation}
    \label{eqn:important_sampling}
    \begin{aligned}    
    p(\bm{x},t)
    &=\int_{\Omega_0}\frac{p(\bm{x},t|\bm{x}_0)p_0(\bm{x}_0)}{q(\bm{x}_0|\bm{x},t)}q(\bm{x}_0|\bm{x},t)\mathrm{d}\bm{x}_0\\
    &\approx \hat{p}_{\bm{\theta}}(\bm{x},t)\coloneqq\frac{1}{M}\sum_{i=1}^M\frac{p(\bm{x},t|\bm{x}_0^i)p_0(\bm{x}_0^i)}{q(\bm{x}_0^i|\bm{x},t)},\quad \bm{x}_0^i\sim q(\cdot|\bm{x},t).
    \end{aligned}
\end{equation}

Recall that given an initial state \(\bm{x}_0\) and time \(t\), \(\bm{x}\) concentrates around the initial state \(\bm{x}_0\) for small time. The corresponding base distribution used in our conditional normalizing flow is
\begin{equation*}
    p_{\widehat{X}_t}(\bm{x},t|\bm{x}_0)=\mathcal{N}\left(\bm{x}; \bm{x_0}+\int_0^te^{\nabla\bm{f}(\bm{x}_0)(t-s)}\bm{f}(\bm{x}_0)\mathrm{d}s, \bm{\Sigma}[\widehat{\bm{X}}_t|\widehat{\bm{X}}_0=\bm{x}_0]\right).
\end{equation*}
Let
\begin{equation}
    \label{eqn:important_distribution_base}
    q_1(\bm{x}_0|\bm{x},t)=\mathcal{N}(\bm{x}_0; m(\bm{x},t), \tilde{\bm{\Sigma}}(\bm{x},t)),
\end{equation}
where \(m(\bm{x},t)=\bm{x}-\int_0^te^{\nabla\bm{f}(\bm{x})s}\bm{f}(\bm{x})\mathrm{d}s\), \(\tilde{\bm{\Sigma}}(\bm{x},t)=\int_0^te^{\nabla\bm{f}(\bm{x})s}\bm{g}(\bm{x})\bm{g}(\bm{x})^\top e^{\nabla \bm{f}(\bm{x})^\top s}\mathrm{d}s\). The distribution \(q_1\) is primarily designed to capture the local concentration of the transition PDF and to mitigate the effect of large variance when \(t\) is small. However, the approximation accuracy of \(q_1\) deteriorates as \(t\) increases. As \(t\) grows, the transition PDF becomes increasingly spread out and therefore does not hinder the approximation of \(p(\bm{x},t)\). This motivates incorporating the initial distribution into the proposal. Accordingly, we propose to use a mixture of
\(q_1(\cdot|\bm{x},t)\) and \(p_0(\cdot)\) as the proposal distribution in the importance sampling scheme. Define
\begin{equation}
    \label{eqn:mixture_proposal_t}
    q_2(\bm{x}_0|\bm{x},t)=\alpha(t)q_1(\bm{x}_0|\bm{x},t)+(1-\alpha(t))p_0(\bm{x}_0),
\end{equation}
where \(\alpha(t)\in(0,1)\) is decreasing with respect to \(t\), e.g., \(\alpha(t)=\exp(-at), a>0.\) The set of samples \(\{\bm{x}_0^i\}_{i=1}^M\sim q_2(\bm{x}_0|\bm{x},t)\) can be generated by drawing \(\alpha(t)M\) samples from \(q_1\) and \((1-\alpha(t))M\) samples from \(p_0\).

\section{Numerical experiments}\label{sec:experiments}

In this section, we present three numerical experiments to demonstrate the effectiveness of the proposed approach. The conditional normalizing flow is applied to solve the transition FPE under various initial states. Once the transition density has been learned, we use importance sampling to evaluate \(p(\bm{x},t)\) given an initial distribution. We first consider a benchmark problem where the analytical solution of the transition PDF is available. 
We next apply our method to an SDE with nonlinear drift and constant diffusion. To further demonstrate the robustness, we investigate the case of state-dependent diffusion. 

All numerical experiments are implemented in PyTorch and run on a single NVIDIA H100 GPU. All models are trained using the Adam optimizer with an initial learning rate of \(10^{-3}\). The learning-rate schedules and
other training parameters are specified for each experiment. The number of quadrature points in the equation \eqref{Gauss_integral} is set to 10.

We evaluate the proposed method using relative \(L^2\) errors and the maximum mean discrepancy (MMD). When the exact transition density is available, we measure the approximation error at a fixed time \(t>0\) by
\begin{equation}
    \label{eqn:relative_L2_transition_PDF_continuous}
    \operatorname{Rel}(p_{\text{C-KRnet}}(\cdot,t|\cdot))
    =
    \left[
    \frac{
    \int_{\Omega_0}\int_{\Omega}
    \left|p(\bm{x},t|\bm{x}_0)
    -p_{\text{C-KRnet}}(\bm{x},t|\bm{x}_0)\right|^2
    \,\mathrm{d}\bm{x}\,\mathrm{d}\bm{x}_0
    }{
    \int_{\Omega_0}\int_{\Omega}
    p(\bm{x},t|\bm{x}_0)^2
    \,\mathrm{d}\bm{x}\,\mathrm{d}\bm{x}_0
    }
    \right]^{1/2}.
\end{equation}
We estimate \eqref{eqn:relative_L2_transition_PDF_continuous}
by importance sampling with proposal PDF \(\rho_t(\bm{x}_0,\bm{x})\):
\begin{equation}
    \label{eqn:relative_L2_transition_PDF}
    \widehat{\operatorname{Rel}}
    (p_{\text{C-KRnet}}(\cdot,t|\cdot))
    =
    \left[
    \frac{
    \displaystyle\sum_{i=1}^{N_v}
    \frac{
    \left|p(\bm{x}^i,t|\bm{x}_0^i)
    -p_{\text{C-KRnet}}(\bm{x}^i,t|\bm{x}_0^i)\right|^2
    }{\rho_t(\bm{x}_0^i,\bm{x}^i)}
    }{
    \displaystyle\sum_{i=1}^{N_v}
    \frac{p(\bm{x}^i,t|\bm{x}_0^i)^2}
    {\rho_t(\bm{x}_0^i,\bm{x}^i)}
    }
    \right]^{1/2},
\end{equation}
where \((\bm{x}_0^i,\bm{x}^i)\), \(i=1,\ldots,N_v\),
are independent samples from \(\rho_t\).
For \(d=2\), we choose a bounded domain \(\Omega\)
that contains most of the probability mass of the transition
densities under consideration and sample from the uniform distribution \(\rho_t(\bm{x}_0,\bm{x})
    =\frac{1}{|\Omega_0\times\Omega|}\mathbf{1}_{\Omega_0\times\Omega}(\bm{x}_0,\bm{x})\). 
For \(d>2\), we set \(\Omega=\mathbb{R}^d\) and use
\[
    \rho_t(\bm{x}_0,\bm{x})
    =\frac{1}{|\Omega_0|}p(\bm{x},t|\bm{x}_0)\mathbf{1}_{\Omega_0}(\bm{x}_0),
\]
where \(|\Omega_0|\) denotes the volume of \(\Omega_0\). 

For a prescribed initial density \(p_0\), we assess the reconstructed solution \(\hat{p}_{\bm{\theta}}(\bm{x},t)\) using its relative \(L^2\) error with respect to a reference solution \(p(\bm{x},t)\):
\begin{equation}
    \label{eqn:relative_L2_PDF_continuous}
    \operatorname{Rel}(\hat{p}_{\bm{\theta}}(\cdot,t))
    =
    \left[
    \frac{
    \int_{\Omega}
    \left|p(\bm{x},t)-\hat{p}_{\bm{\theta}}(\bm{x},t)\right|^2
    \,\mathrm{d}\bm{x}
    }{
    \int_{\Omega}p(\bm{x},t)^2\,\mathrm{d}\bm{x}
    }
    \right]^{1/2}.
\end{equation}
As above, we estimate this error by importance sampling:
\begin{equation}
    \label{eqn:relative_L2_PDF}
    \widehat{\operatorname{Rel}}
    (\hat{p}_{\bm{\theta}}(\cdot,t))
    =
    \left[
    \frac{
    \displaystyle\sum_{i=1}^{N_v}
    \frac{
    \left|p(\bm{x}^i,t)-\hat{p}_{\bm{\theta}}(\bm{x}^i,t)\right|^2
    }{\rho_t(\bm{x}^i)}
    }{
    \displaystyle\sum_{i=1}^{N_v}
    \frac{p(\bm{x}^i,t)^2}{\rho_t(\bm{x}^i)}
    }
    \right]^{1/2},
\end{equation}
where \(\bm{x}^i\), \(i=1,\ldots,N_v\), are independent
samples from the proposal density \(\rho_t(\bm{x})\). For \(d=2\), we use a bounded domain \(\Omega\) containing most of the solution's probability mass and sample uniformly from \(\Omega\). For \(d>2\), we set
\(\Omega=\mathbb{R}^d\) and use \(\rho_t(\bm{x})=p(\bm{x},t)\). When the transition density is available analytically, we compute the reference solution for the prescribed initial density by evaluating the Chapman-Kolmogorov integral using Gauss-Kronrod quadrature.
For two-dimensional problems without an analytical solution
and with continuous initial densities on \(\mathbb{R}^2\),
we compute the reference solution using an alternating-direction implicit (ADI) finite difference scheme \cite{pichler2013numerical}, with time step
\(\Delta t=0.001\) and spatial mesh width \(h=0.01\).
For discontinuous initial densities, we additionally assess the reconstructed distributions using the MMD
\cite{gretton2012kernel}.
We use the empirical squared MMD
\begin{equation}
    \label{eqn:MMD}
    \widehat{\operatorname{MMD}}^{\,2}(t)
    =
    \frac{1}{M_1^2}
    \sum_{i=1}^{M_1}\sum_{j=1}^{M_1}
    K(\bm{x}_t^i,\bm{x}_t^j)
    -
    \frac{2}{M_1M_2}
    \sum_{i=1}^{M_1}\sum_{j=1}^{M_2}
    K(\bm{x}_t^i,\bm{y}_t^j)
    +
    \frac{1}{M_2^2}
    \sum_{i=1}^{M_2}\sum_{j=1}^{M_2}
    K(\bm{y}_t^i,\bm{y}_t^j),
\end{equation}
with the Gaussian mixture kernel
\begin{equation*}
    K(\bm{x},\bm{y})
    =
    \frac{1}{3}
    \exp\left(-\frac{4\|\bm{x}-\bm{y}\|^2}{\hat{\sigma}^2}\right)
    +
    \frac{1}{3}
    \exp\left(-\frac{\|\bm{x}-\bm{y}\|^2}{\hat{\sigma}^2}\right)
    +
    \frac{1}{3}
    \exp\left(-\frac{\|\bm{x}-\bm{y}\|^2}{4\hat{\sigma}^2}\right).
\end{equation*}
The reference samples \(\{\bm{x}_t^i\}_{i=1}^{M_1}\)
are obtained using the Euler--Maruyama scheme with
time step \(0.001\), whereas
\(\{\bm{y}_t^j\}_{j=1}^{M_2}\) are generated by the learned
conditional normalizing flow. Both sample sets use initial
states drawn from \(p_0\).
The bandwidth \(\hat{\sigma}\) is the median of the
pairwise distances between the two sample sets.

\subsection{Bene\v{s} SDE}
We begin by demonstrating the performance of our method on the Bene\v{s} SDE, 
\begin{equation}
\label{eqn:ex_Benes_SDE}
\left\{\begin{split}
&\mathrm{d}\bm{X}_t=
	\tanh \bm{X}_t\mathrm{d}t+\bm{I}_d\mathrm{d}\bm{W}_t,\\
 &   \bm{X}_0=\bm{x}_0, \quad \bm{x}_0\in[-1,1]^d,
\end{split}\right.
\end{equation}
for which the corresponding transition PDF admits an analytical solution,
\begin{equation*}
p(\bm{x},t|\bm{x}_{0})=\frac{1}{(2\pi t)^{d/2}}\exp\left(-\frac{d}{2}t\right)\exp\left(-\frac{1}{2t}\left\|\bm{x}-\bm{x}_0\right\|^2\right)\prod_{i=1}^d\frac{\cosh x_i}{\cosh x_{0,i}}.
\end{equation*}

\subsubsection{Two-dimensional case}

 We first consider the two-dimensional case. The conditional normalizing flow is composed of eight coupling layers. The parameters of each coupling layer are generated by a neural network with 32 random Fourier features and two hidden layers with 32 neurons. Training is performed for 10 adaptive iterations, each with 300 epochs. The learning rate is halved every 2000 epochs. \(2\times 10^5\) training points are employed with a batch size of \(10^4\) each epoch. A test set of $10^6$ samples is used to compute relative errors throughout the training process.

We first investigate the approximation performance of the transition PDF for various initial states, using the following three types of loss functions:

\begin{enumerate}
    \item Vanilla PINN, \(t\in(0, 1.5)\), \(\gamma_1=0.2, \gamma_2=0.6, \gamma_3=0.2\);
    \item Vanilla PINN, \(t\in (0.1, 1.5)\), \(\gamma_1=0.2, \gamma_2=0.6, \gamma_3=0.2\);
    \item Time-weighted loss function, \(t\in(0, 1.5)\), \(\gamma_1=0.2, \gamma_2=0.6, \gamma_3=0.2\).
\end{enumerate}

\begin{figure}[htbp]
	\centering
	\begin{minipage}[b]{0.34\linewidth}
		\includegraphics[height=4.3cm,width=5.2cm]{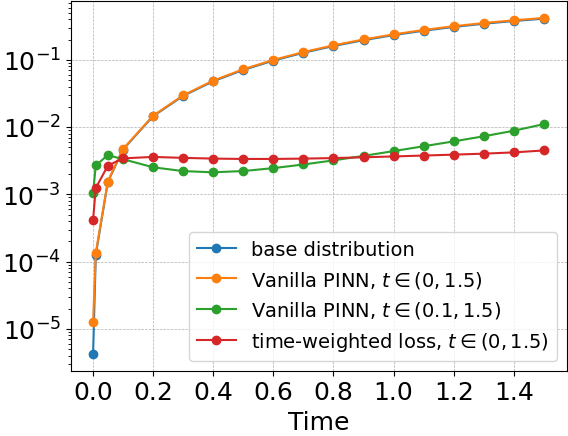}
		\subcaption{The relative \(L^2\) error.}
        \label{fig:Benes_relative_error_comp}
	\end{minipage}
	\begin{minipage}[b]{0.34\linewidth}
		\includegraphics[height=4.3cm,width=5.2cm]{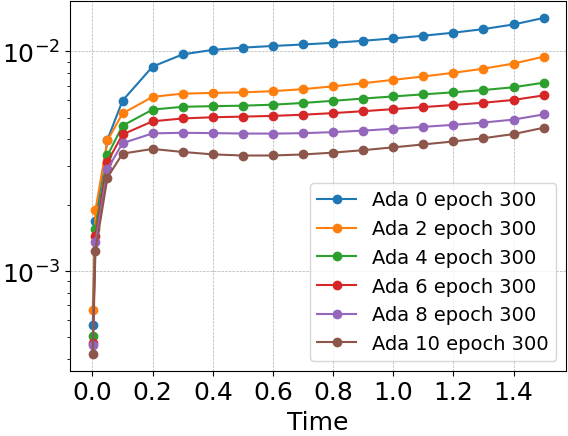}
		\subcaption{Time-weighted loss.}
        \label{fig:Benes_relative_error_adapt}
	\end{minipage}
    \caption{Left: \(\widehat{\text{Rel}}(p_{\text{C-KRnet}}(\cdot,t|\cdot))\) for different loss functions. Right:  \(\widehat{\text{Rel}}(p_{\text{C-KRnet}}(\cdot,t|\cdot))\) over time at different adaptive iterations.}
\end{figure}

We present the relative \(L^2\) error \(\widehat{\text{Rel}}(p_{\text{C-KRnet}}(\cdot,t|\cdot))\) \eqref{eqn:relative_L2_transition_PDF} in Fig.~\ref{fig:Benes_relative_error_comp}, where the test set \(\{\bm{x}_0,\bm{x}\}\) consists of \(10^5\) samples drawn uniformly from \([-1,1]^2\times[-5,5]^2\). The blue curve, which overlaps with the orange one, represents the relative error of the base distribution. While it achieves high accuracy as \(t\to 0\), its approximation quality deteriorates as \(t\) increases. A similar trend is observed when applying vanilla PINN. Although the base distribution stays close to the exact solution and the conditional normalizing flow remains close to the identity at small times, the corresponding residual at small times may still dominate, so the approximation accuracy is nearly the same as that of the base distribution.
If we exclude small times by setting the minimum training time to \(t=0.1\), the PDE constraints are no longer imposed near the initial condition. In this case, the approximation relies primarily on the continuity of the flow, and the accuracy improves for \(t>0.1\). By contrast, our method exploits the information from the base distribution and provides accurate approximations even as \(t\) increases, owing to the expressiveness of the conditional normalizing flow and the more balanced PINN  loss. 
Notably, we do not observe the typical error growth with time that often occurs in time-dependent problems. The weighted formulation maintains low relative errors over the tested time interval. Fig.~\ref{fig:Benes_relative_error_adapt} further illustrates the relative \(L^2\) error \(\widehat{\text{Rel}}(p_{\text{C-KRnet}}(\cdot,t|\cdot))\)
at different adaptive iterations for the time-weighted loss function. Leveraging samples from the solution itself allows the model to capture the local behavior of the solution and enhances training efficiency. 

\begin{figure}[htbp]
\begin{minipage}[b]{0.185\linewidth}
    \begin{overpic}[height=2.5cm,width=2.9cm]{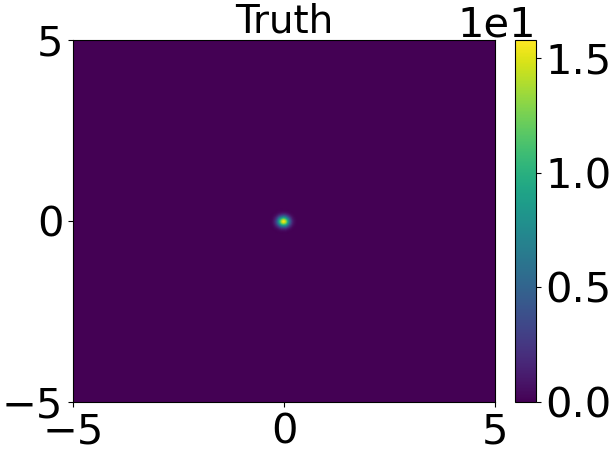}
  \end{overpic}
\end{minipage}
\begin{minipage}[b]{0.185\linewidth}
     \begin{overpic}[height=2.5cm,width=2.9cm]{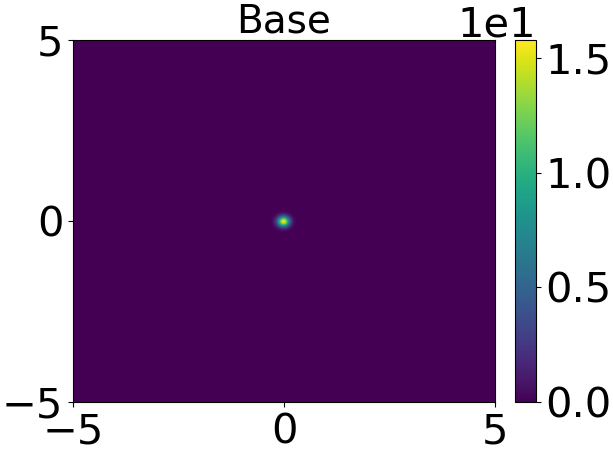}
  \end{overpic}
\end{minipage}
\begin{minipage}[b]{0.185\linewidth}
    \begin{overpic}[height=2.5cm,width=2.9cm]{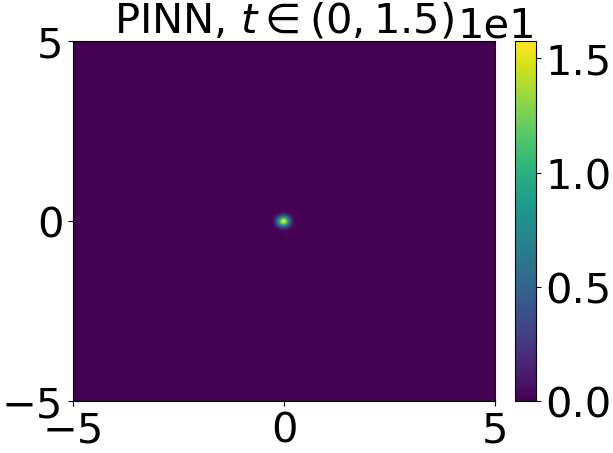}
  \end{overpic}
\end{minipage}
\begin{minipage}[b]{0.185\linewidth}
     \begin{overpic}[height=2.5cm,width=2.9cm]{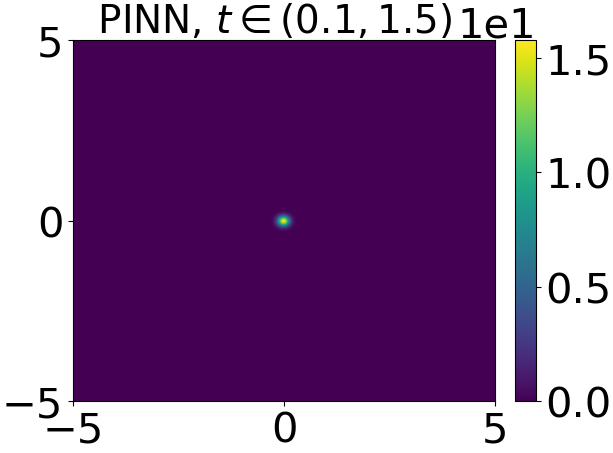}
  \end{overpic}
\end{minipage}
\begin{minipage}[b]{0.185\linewidth}
     \begin{overpic}[height=2.5cm,width=2.9cm]{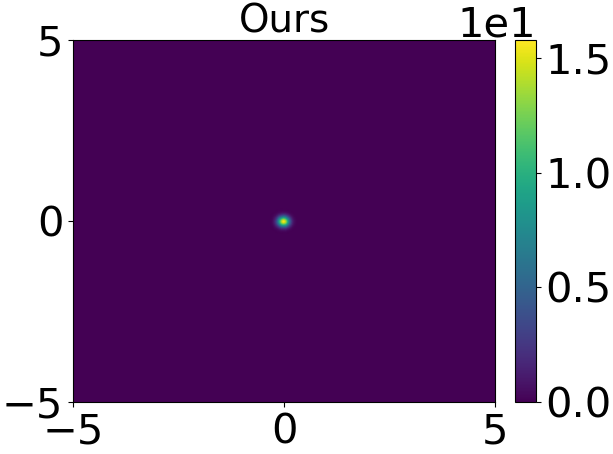}
  \end{overpic}
\end{minipage}

\hspace{90pt}
\begin{minipage}[b]{0.185\linewidth}
    \includegraphics[height=2.5cm,width=2.7cm]{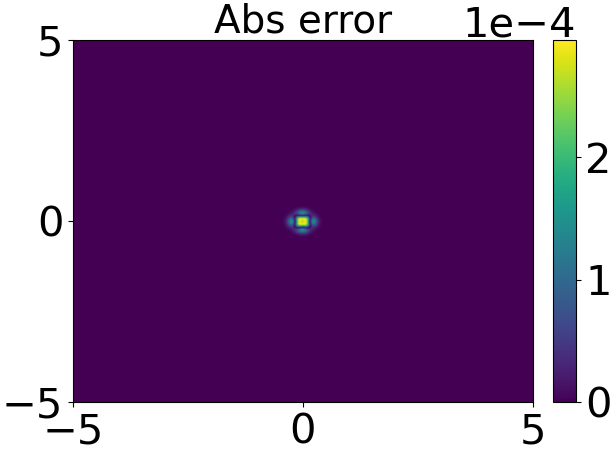}
\end{minipage}
\begin{minipage}[b]{0.185\linewidth}
    \includegraphics[height=2.5cm,width=2.7cm]{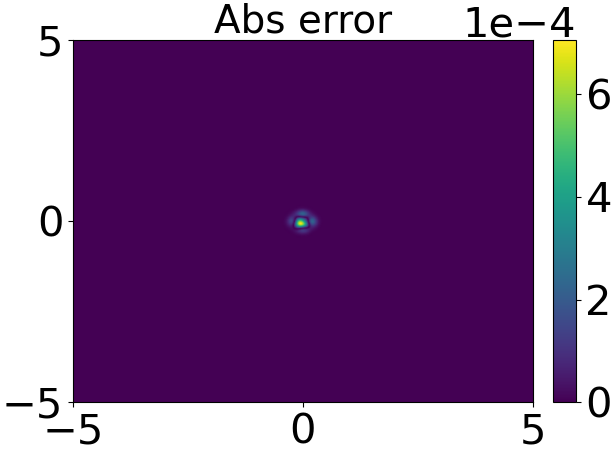}
\end{minipage}
\begin{minipage}[b]{0.185\linewidth}
    \includegraphics[height=2.5cm,width=2.7cm]{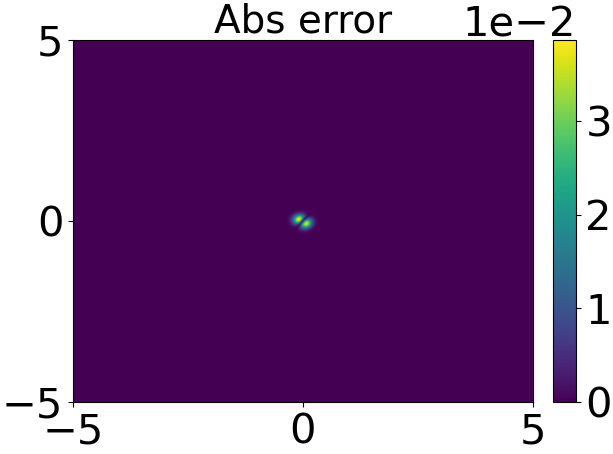}
\end{minipage}
\begin{minipage}[b]{0.185\linewidth}
    \includegraphics[height=2.5cm,width=2.9cm]{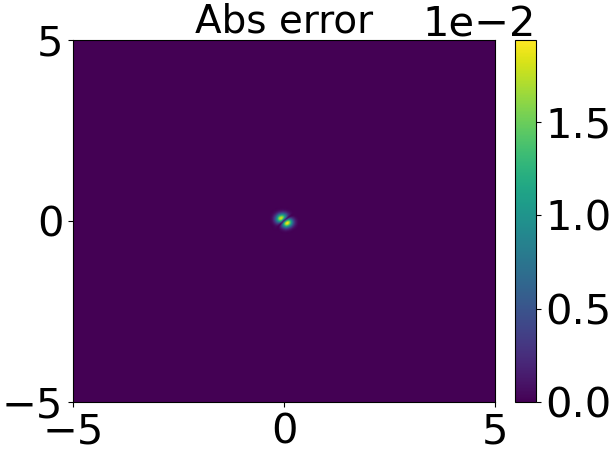}
\end{minipage}
\caption{Solutions (top row) and the absolute errors (bottom row) at \(t=0.01\). The first column shows the ground-truth solution. Columns 2-5 display numerical approximations obtained by, from left to right, the base distribution, Vanilla PINN with \(t\in(0,1.5)\), Vanilla PINN with \(t\in(0.1,1.5)\), time-weighted loss function with \(t\in(0,1.5)\).}
\label{fig:Benes_sol_t001}
\end{figure}

\begin{figure}[htbp]

\begin{minipage}[b]{0.185\linewidth}
    \begin{overpic}[height=2.5cm,width=3cm]{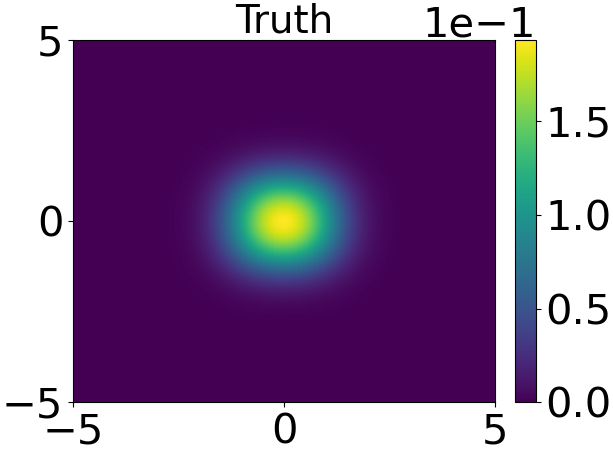}
  \end{overpic}
\end{minipage}
\begin{minipage}[b]{0.185\linewidth}
     \begin{overpic}[height=2.5cm,width=2.9cm]{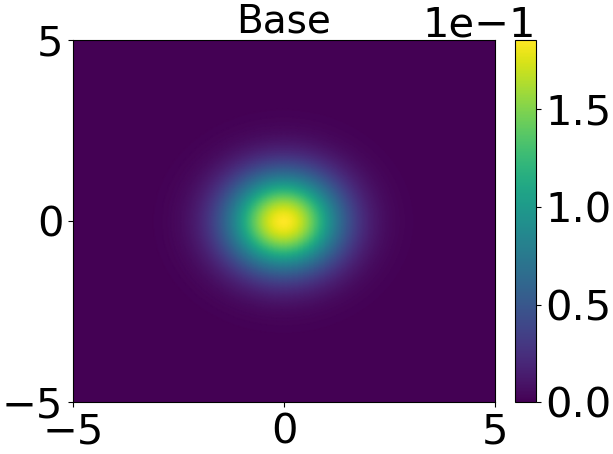}
  \end{overpic}
\end{minipage}
\begin{minipage}[b]{0.185\linewidth}
    \begin{overpic}[height=2.5cm,width=2.9cm]{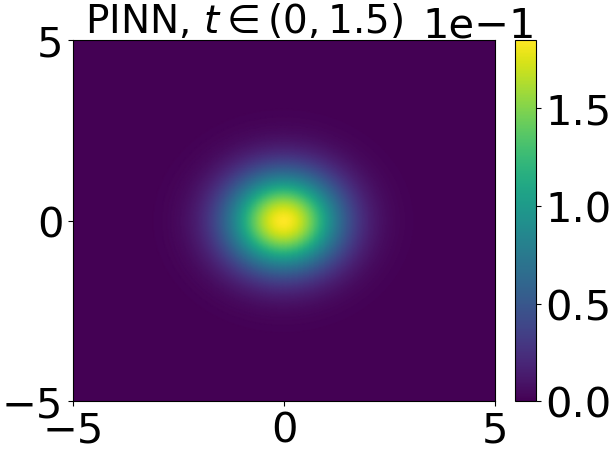}
  \end{overpic}
\end{minipage}
\begin{minipage}[b]{0.185\linewidth}
     \begin{overpic}[height=2.5cm,width=2.9cm]{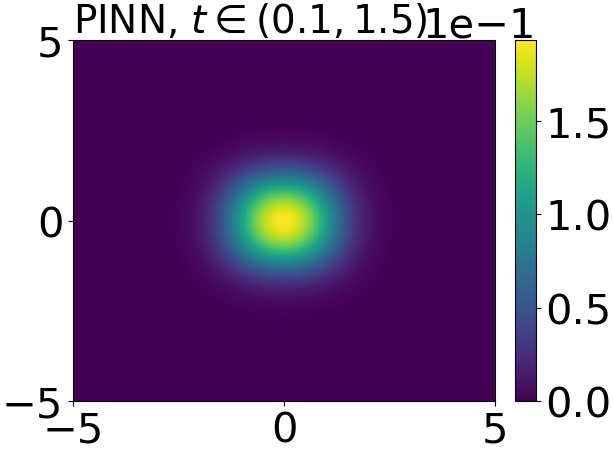}
  \end{overpic}
\end{minipage}
\begin{minipage}[b]{0.185\linewidth}
     \begin{overpic}[height=2.5cm,width=2.9cm]{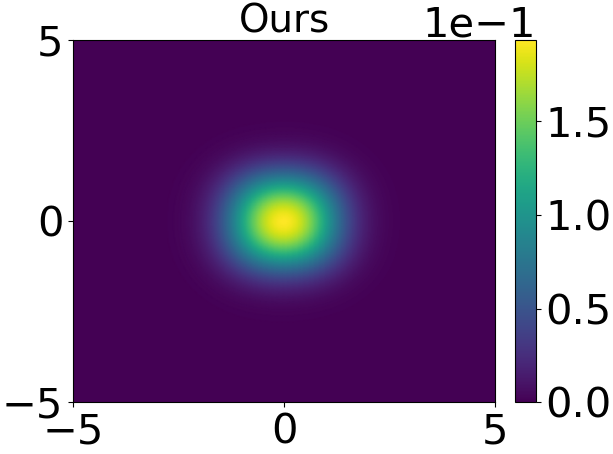}
  \end{overpic}
\end{minipage}

\hspace{90pt}
\begin{minipage}[b]{0.185\linewidth}
    \includegraphics[height=2.5cm,width=2.95cm]{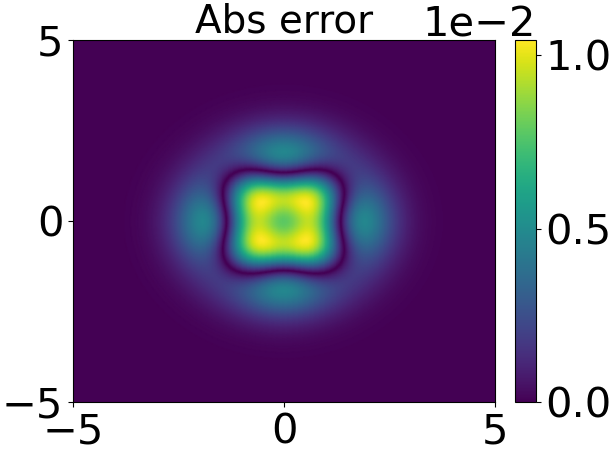}
\end{minipage}
\begin{minipage}[b]{0.185\linewidth}
    \includegraphics[height=2.5cm,width=2.95cm]{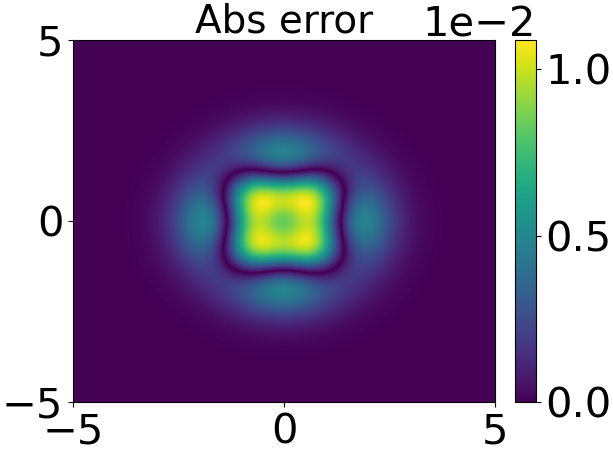}
\end{minipage}
\begin{minipage}[b]{0.185\linewidth}
    \includegraphics[height=2.5cm,width=2.7cm]{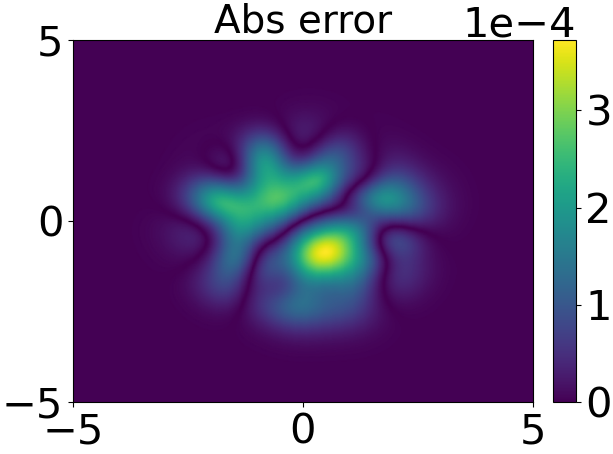}
\end{minipage}
\begin{minipage}[b]{0.185\linewidth}
    \includegraphics[height=2.5cm,width=2.9cm]{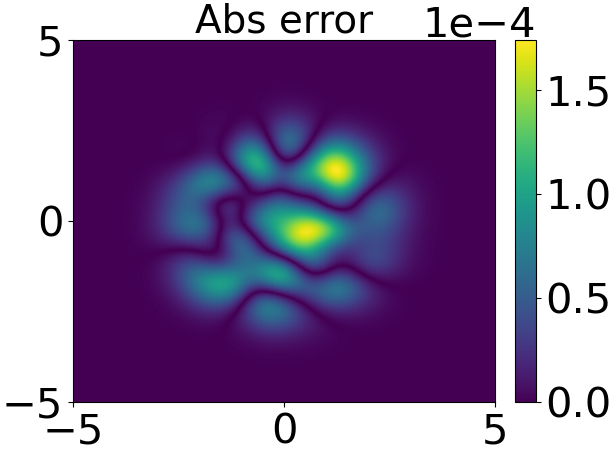}
\end{minipage}
\caption{Same as Fig.~\ref{fig:Benes_sol_t001} for \(t=0.5\).}
\label{fig:Benes_sol_t05}
\end{figure}

\begin{figure}[htbp]
\begin{minipage}[b]{0.185\linewidth}
    \begin{overpic}[height=2.5cm,width=2.7cm]{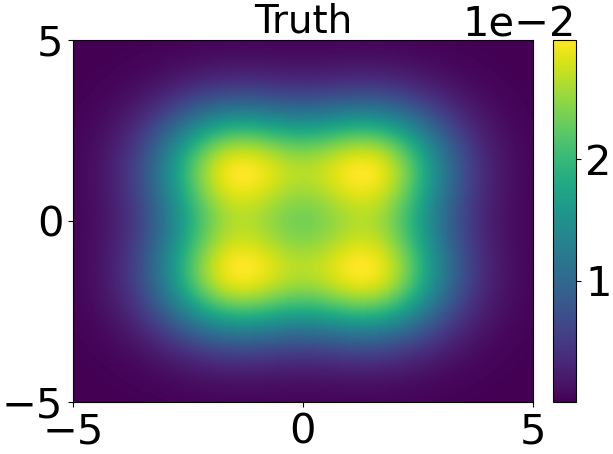}
  \end{overpic}
\end{minipage}
\begin{minipage}[b]{0.185\linewidth}
     \begin{overpic}[height=2.5cm,width=2.9cm]{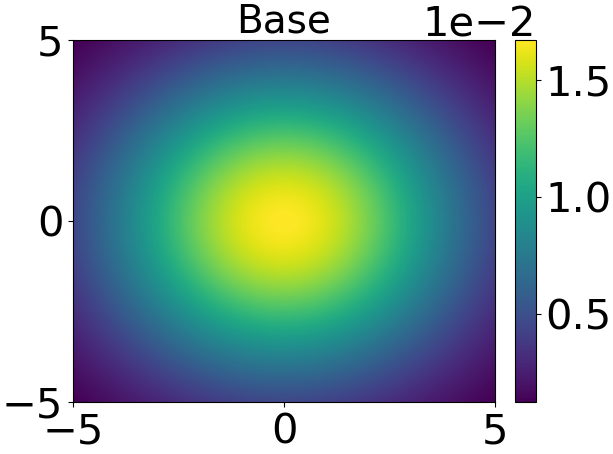}
  \end{overpic}
\end{minipage}
\begin{minipage}[b]{0.185\linewidth}
    \begin{overpic}[height=2.5cm,width=2.9cm]{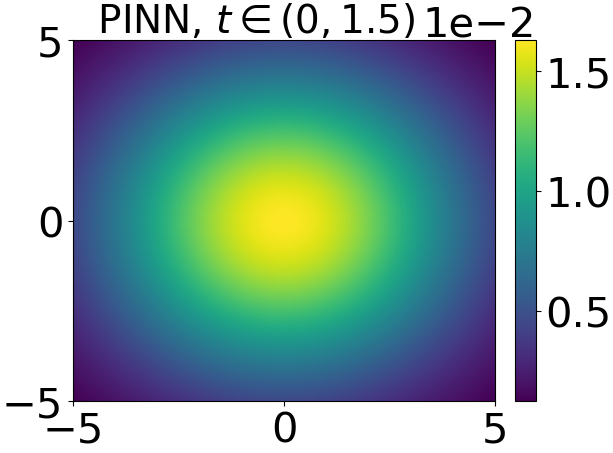}
  \end{overpic}
\end{minipage}
\begin{minipage}[b]{0.185\linewidth}
     \begin{overpic}[height=2.5cm,width=2.7cm]{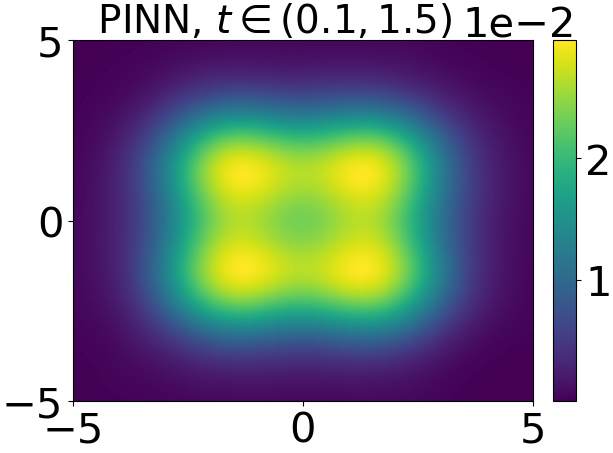}
  \end{overpic}
\end{minipage}
\begin{minipage}[b]{0.185\linewidth}
     \begin{overpic}[height=2.5cm,width=2.7cm]{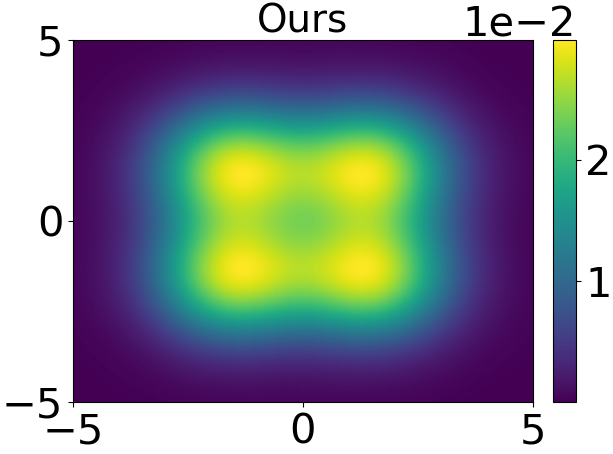}
  \end{overpic}
\end{minipage}

\hspace{90pt}
\begin{minipage}[b]{0.185\linewidth}
    \includegraphics[height=2.5cm,width=2.9cm]{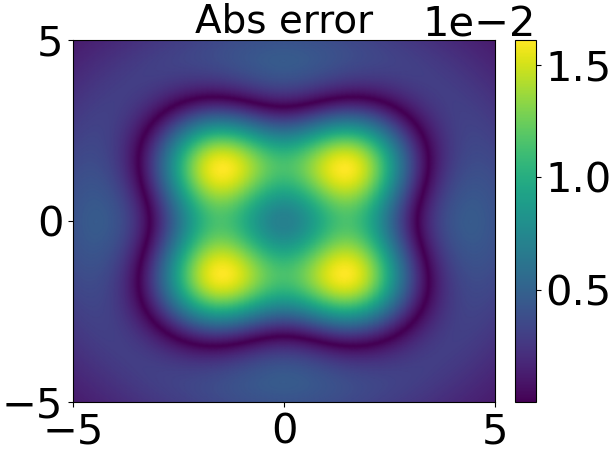}
\end{minipage}
\begin{minipage}[b]{0.185\linewidth}
    \includegraphics[height=2.5cm,width=2.9cm]{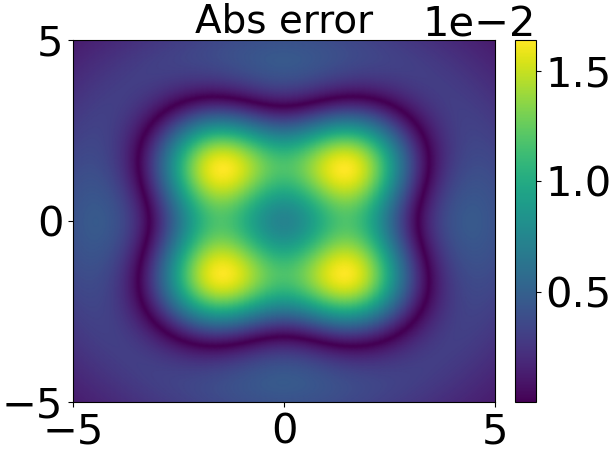}
\end{minipage}
\begin{minipage}[b]{0.185\linewidth}
    \includegraphics[height=2.5cm,width=2.7cm]{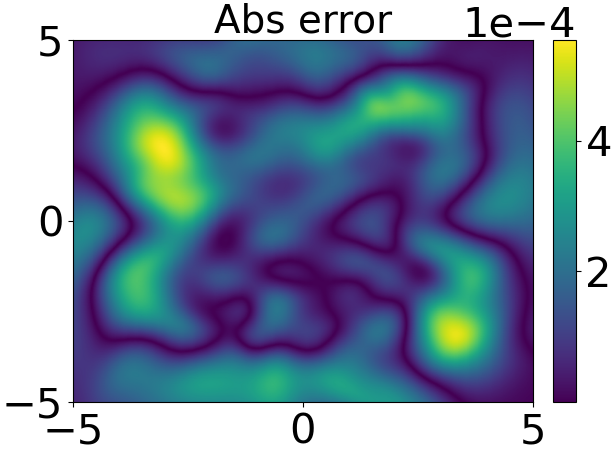}
\end{minipage}
\begin{minipage}[b]{0.185\linewidth}
    \includegraphics[height=2.5cm,width=2.9cm]{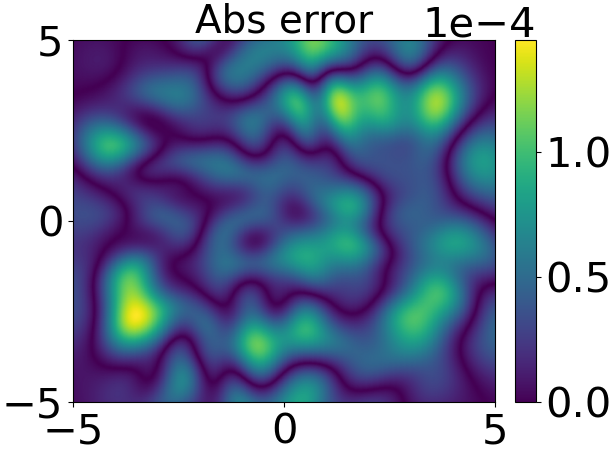}
\end{minipage}
\caption{Same as Fig.~\ref{fig:Benes_sol_t001} for \(t=1.5\).}
\label{fig:Benes_sol_t15}
\end{figure}

Fig.~\ref{fig:Benes_sol_t001}, Fig.~\ref{fig:Benes_sol_t05}, and Fig.~\ref{fig:Benes_sol_t15} present the solutions at different times with the initial state \(\bm{x}_0=(0,0)\) computed with the different loss functions. They all agree closely with the exact solution except for the vanilla PINN's loss for \(t\in(0,1.5)\). While the base distribution and vanilla PINNs yield much more accurate approximations for small \(t\), their performance degrades progressively as \(t\) increases. In particular, Fig.~\ref{fig:Benes_sol_t15} shows that the conditional normalizing flow is capable of capturing multi-modal solutions. These results confirm the effectiveness of the proposed approach.

\begin{figure}[htbp]
	\begin{minipage}[b]{0.325\linewidth}
		\includegraphics[height=4.4cm,width=5.2cm]{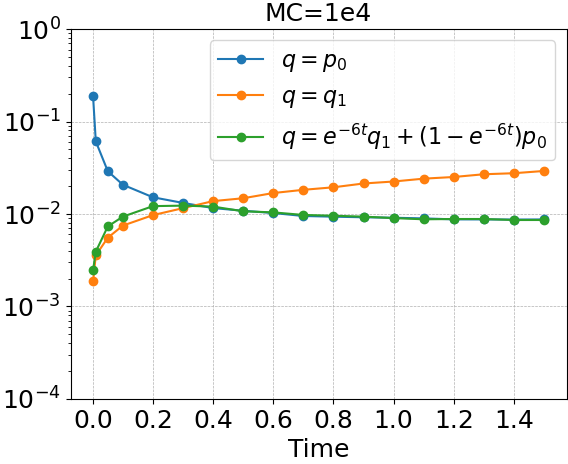}
	\end{minipage}
    \begin{minipage}[b]{0.325\linewidth}
		\includegraphics[height=4.4cm,width=5.2cm]{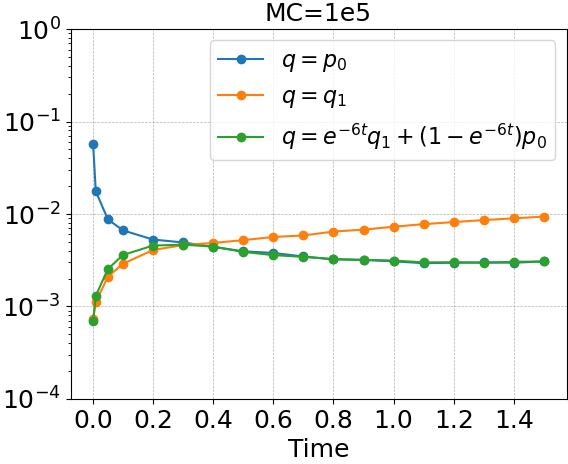}
	\end{minipage}
	\begin{minipage}[b]{0.325\linewidth}
		\includegraphics[height=4.4cm,width=5.2cm]{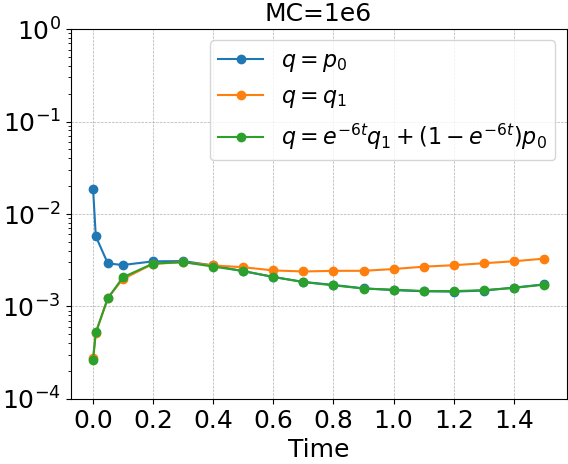}
	\end{minipage}
    \caption{Uniform case: Evolution of \(\widehat{\text{Rel}}(\hat{p}_{\bm{\theta}}(\cdot,t))\) over time. The subplots are arranged from left to right with \(M=10^4, 10^5, 10^6\).}
    \label{fig:Benes_relative_error_uni_MC_t}
\end{figure}

\begin{figure}[htbp]
\centering
\begin{minipage}[b]{0.325\linewidth}
		\includegraphics[height=4.4cm,width=5.2cm]{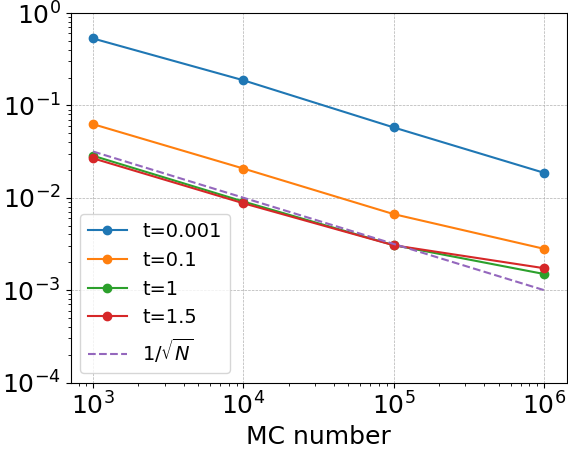}
        \subcaption{\(q=p_0\)}
	\end{minipage}
	\begin{minipage}[b]{0.325\linewidth}
		\includegraphics[height=4.4cm,width=5.2cm]{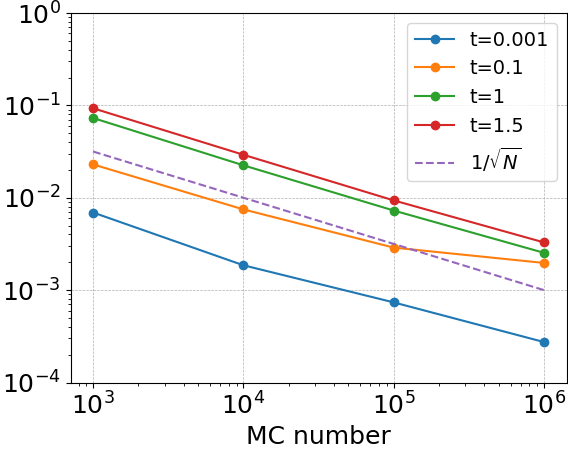}
        \subcaption{\(q=q_1\)}
	\end{minipage}
    \begin{minipage}[b]{0.325\linewidth}
		\includegraphics[height=4.4cm,width=5.2cm]{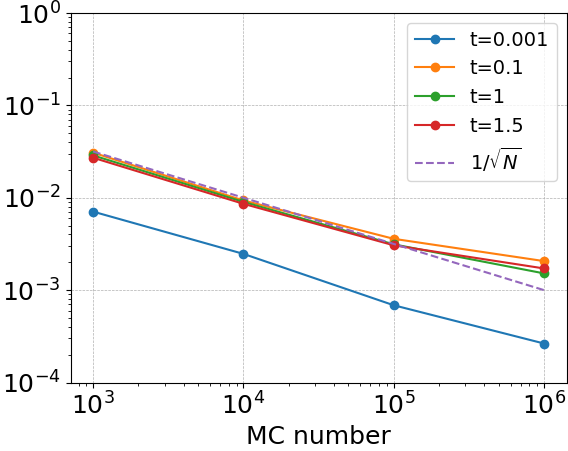}
        \subcaption{\(q=e^{-6t}q_1+(1-e^{-6t})p_0\)}
	\end{minipage}
    \caption{Uniform case: The decay of \(\widehat{\text{Rel}}(\hat{p}_{\bm{\theta}}(\cdot,t))\) with respect to the number of Monte Carlo samples at different times. Left: Draw samples from \(p_0\); Middle: Draw samples from \(q_1\); Right: Draw samples from the mixture of \(p_0\) and \(q_1\).}
    \label{fig:Benes_relative_error_uni_MC}
\end{figure}

We next assess how well the solution of the FPE \(p(\bm{x},t)\) can be approximated under various initial conditions. 
To this end, we employ a variety of Beta distributions on \([-1,1]^d\) as initial conditions, thereby evaluating the effectiveness of the conditional normalizing flow for operator learning problems. More precisely, 
\begin{equation*}
\textbf{Beta}(\bm{x};\bm{\zeta},\bm{\eta})=\prod_{i=1}^d \frac{1}{2\bm{B}(\zeta_i, \eta_i)}\left(\frac{1+x_i}{2}\right)^{\zeta_i-1}\left(\frac{1-x_i}{2}\right)^{\eta_i-1}, \bm{x}\in [-1,1]^d,
\end{equation*}
where \(\bm{B}(\zeta_i, \eta_i)\) denotes a beta function. We first explore the uniform case where \(\bm{\zeta}=\bm{\eta}=(1,1)\).
We use the conditional normalizing flow trained with the time-weighted loss as an approximation of the transition PDF.  
Fig.~\ref{fig:Benes_relative_error_uni_MC_t} reports the relative \(L^2\) error \(\widehat{\text{Rel}}(p_{\bm{\theta}}(\cdot,t))\) \eqref{eqn:relative_L2_PDF} with a test set consisting of \(10^4\) grid points over \([-5,5]^2\) for different numbers of Monte Carlo samples used in evaluating \(p_{\bm{\theta}}(\bm{x},t)\) \eqref{eqn:important_sampling}. Using \(q_1\) as the proposal distribution yields the most accurate approximation for \(t<0.3\), whereas drawing samples directly from the initial distribution \(p_0\) provides more reliable approximations for \(t\geq 0.4\).  By contrast, adopting \(q_2\), a mixture of \(q_1\) and \(p_0\) with \(\alpha(t) = \exp(-6t)\), offers a balanced compromise, delivering acceptable accuracy across both small and large times. 
Fig.~\ref{fig:Benes_relative_error_uni_MC} further illustrates the error decay with respect to the number of Monte Carlo samples. The relative error follows the Monte Carlo error scaling for \(M\leq 10^5\). For a larger Monte Carlo sample size, however, the error in approximating the transition PDF becomes dominant. We also present the variance of the importance-sampling Monte Carlo estimator \(\hat{p}_{\bm{\theta}}(\bm{x},t)\) in \eqref{eqn:important_sampling}
as a function of \(t\) in Fig.~\ref{fig:Benes_variance} using \(10^4\) samples. The proposal \(q_1\) substantially reduces the variance as \(t\to0\), but its variance increases as \(t\) grows. However, \(q_2\) with different \(\alpha(t)\), which combines \(q_1\) with the initial distribution \(p_0\), yields a more balanced variance behavior over time. We note that traditional methods require special treatment when the initial distribution is discontinuous. Our method accommodates this initial distribution for example uniform distribution, providing reliable approximations.

\begin{figure}[htbp]
	\centering
	\begin{minipage}[b]{0.35\linewidth}
		\includegraphics[height=4.3cm,width=5.2cm]{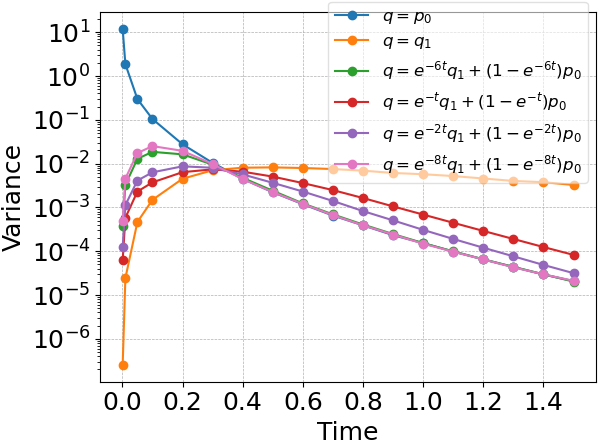}
	\end{minipage}
	\begin{minipage}[b]{0.35\linewidth}
		\includegraphics[height=4.3cm,width=5.2cm]{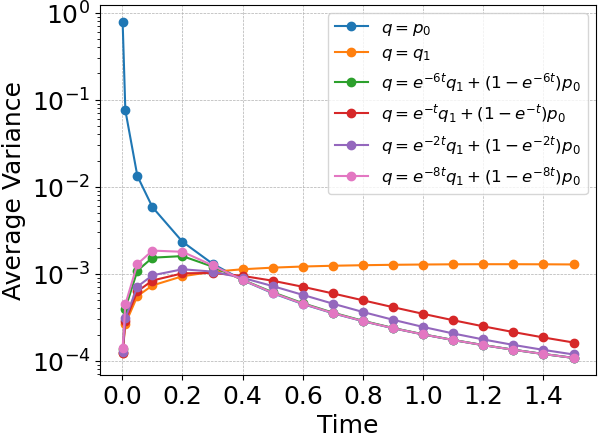}
	\end{minipage}
    \caption{The variance of \(\hat{p}_{\bm{\theta}}(\bm{x},t)\) over time. Left: \(\bm{x}=(0, 0)\). Right: Average variance of the set \(\{\bm{x}^i\}\) consisting of \(10^4\) grid points in \([-5,5]^2\).}  
    \label{fig:Benes_variance}
\end{figure}

Moreover, we investigate the accuracy and the computational cost of the proposed method in a repeated-query setting, where the stochastic dynamics remain fixed while the initial distribution varies. The trained conditional normalizing flow is frozen and reused to reconstruct the solution for each new initial distribution through importance sampling \eqref{eqn:important_sampling} with proposal distribution \(q_2\) \eqref{eqn:mixture_proposal_t} and \(\alpha(t)=\exp(-6t)\).
We present in Fig.~\ref{fig:Benes_various_beta_err} the time evolution of the relative \(L^2\) error \(\widehat{\text{Rel}}(p_{\bm{\theta}}(\cdot,t))\) for Beta distributions chosen as initial conditions. The numbers of Monte Carlo samples used in evaluating \(\hat{p}_{\bm{\theta}}(\bm{x},t)\) are set to be \(10^4\) and \(10^6\). The results confirm that the proposed method yields reliable approximations.

\begin{figure}[htbp]
	\centering
	\begin{minipage}[b]{0.35\linewidth}
		\includegraphics[height=4.4cm,width=5.2cm]{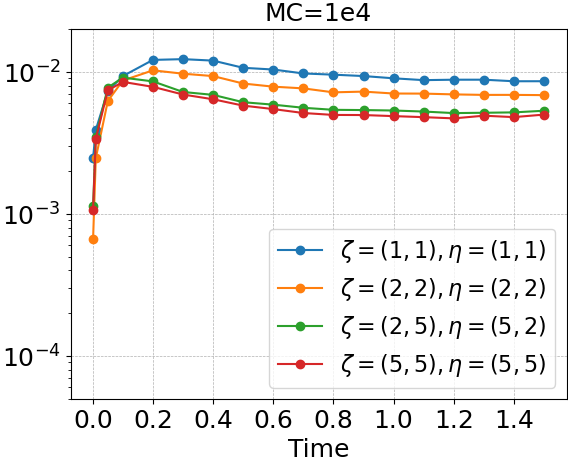}
	\end{minipage}
	\begin{minipage}[b]{0.35\linewidth}
		\includegraphics[height=4.4cm,width=5.2cm]{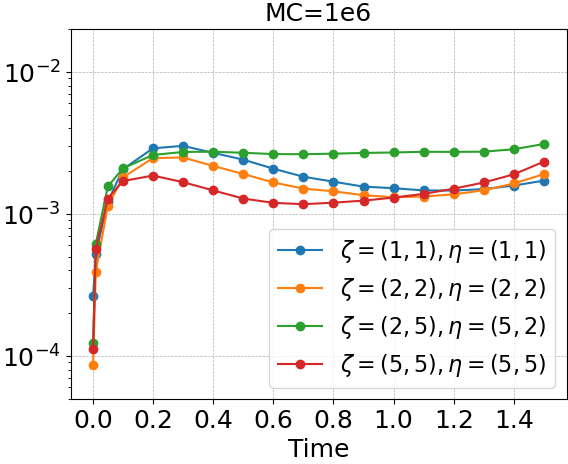}
	\end{minipage}
    \caption{Time evolution of the relative \(L^2\) error \(\widehat{\text{Rel}}(\hat{p}_{\bm{\theta}}(\cdot,t))\) for various initial distributions. Left: $M=10^4$; Right: $M=10^6$.}
    \label{fig:Benes_various_beta_err}
\end{figure}

We also compare the computational cost and accuracy of the proposed method with those of a temporal normalizing-flow solver trained seperately for each of six Beta initial distributions. Both methods are evaluated at the same spatial points and eight output times, \(t=0.001, 0.1, 0.25, 0.5, 0.75, 1, 1.25, 1.5\).  For the operator method, the reported runtime is the total online reconstruction time over these eight times, averaged over ten repetitions with different Monte Carlo random seeds while keeping the trained transition-density model fixed. Each reconstruction uses \(10^4\) samples. The reported mean error is the arithmetic average of the spatial relative $L^2$ errors over the eight output times and ten repetitions. For the single-FPE baseline, we use the temporal normalizing flow of \cite{feng2022solving}, with random Fourier features incorporated into the neural networks defining the affine coupling layers. Its reported runtime includes training from scratch and evaluation at all eight output times, with training accounting for most of the cost. Implementation details of the single-FPE solver are provided in Appendix~\ref{appendix:numerical_details_of_single_solver}.
As shown in Table~\ref{tab:computational_cost}, the operator method achieves lower mean relative \(L^2\) errors than the single-FPE baseline for all six initial distributions under the prescribed training settings. The online reconstruction time remains nearly constant, and the mean errors range from \(5.45\times10^{-3}\) to \(8.96\times 10^{-3}\). A possible explanation for the observed accuracy difference is the distinct treatment of the initial condition. The operator model incorporates the Dirac initial condition through its base distribution and an identity map at \(t=0\). A new initial density then enters through the Chapman-Kolmogorov integral, without training. In contrast, the single-FPE baseline jointly minimizes the PDE residual and a penalty for the initial-distribution mismatch. The operator formulation therefore avoids balancing these two objectives, which may contribute to its lower errors. Reusing the learned transition PDF also amortizes the offline training cost across multiple initial distributions.
Including the offline cost of \(78.87\) minutes, the total cost for the six initial distributions is \(81.23\) minutes, compared with \(501.59\) minutes for six separate baseline solves. Thus, for this set of queries and the prescribed training settings, the operator method reduces the total computational cost by a factor of approximately \(6\).

\begin{table}[t]
\centering
\small
\setlength{\tabcolsep}{4pt}
\begin{tabular}{@{}l cc cc@{}}
\toprule
& \multicolumn{2}{c}{Operator solver}
& \multicolumn{2}{c}{Single-FPE solver} \\
\cmidrule(lr){2-3}
\cmidrule(lr){4-5}
Initial distribution
& Online time (min) &  Error 
& Total time (min) & Error   \\
\midrule
\(\zeta=(1,1),\,\eta=(1,1)\)
& 0.39 & 8.96e-3 & 82.52 & 6.94e-2   \\
\(\zeta=(2,2),\,\eta=(2,2)\)
& 0.40 & 7.00e-3 & 81.65 & 1.25e-2 \\
\(\zeta=(5,5),\,\eta=(5,5)\)
& 0.40 & 5.45e-3 & 82.41 & 8.34e-3   \\
\(\zeta=(2,2),\,\eta=(5,5)\)
& 0.39 & 5.77e-3 & 80.22 & 2.45e-2   \\
\(\zeta=(2,5),\,\eta=(5,2)\)
& 0.39 & 5.65e-3 & 88.20 & 2.53e-2  \\
\(\zeta=(2,5),\,\eta=(2,5)\)
& 0.39 & 6.29e-3 & 86.59 & 1.28e-2   \\
\bottomrule
\end{tabular}
\caption{
Computational cost and accuracy for six Beta initial distributions. The offline time of the operator solver is \(78.87\)~min.}
\label{tab:computational_cost}
\end{table}
 
\subsubsection{Six-dimensional case}

We then solve the six-dimensional FPE with various initial distributions using the proposed time-weighted loss function. Moreover, when evaluating the PDE residual, we scale the solution by a constant factor, set to \(1000\) in the six-dimensional case, to avoid numerical underflow, since the solution of the FPE becomes very small in higher dimensions.
The conditional normalizing flow follows the KRnet structure with descending active transformation dimensions. It consists of five stages with \(10\) coupling layers per stage. The parameters of each coupling layer are determined by a neural network with the same structure as in the two-dimensional case. We apply 20 adaptive iterations, each comprising 300 epochs. The updating rates are set to \(\gamma_1=0.2\), \(\gamma_2=0.6\), and \(\gamma_3=0.2\). The learning rate is halved every 4000 epochs. \(2\times 10^5\) training samples are employed with a batch size of \(2\times 10^4\) per epoch.  

\begin{figure}[htbp]
	\centering
	\begin{minipage}[b]{0.35\linewidth}
		\includegraphics[height=4.4cm,width=5.2cm]{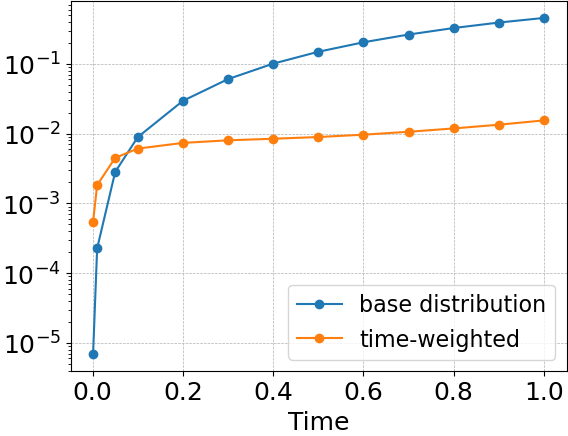}
        \subcaption{\(p(\bm{x},t|\bm{x}_0)\)}
        \label{fig:Benes_high_TPDF_error}
	\end{minipage}
	\begin{minipage}[b]{0.35\linewidth}
		\includegraphics[height=4.5cm,width=5.3cm]{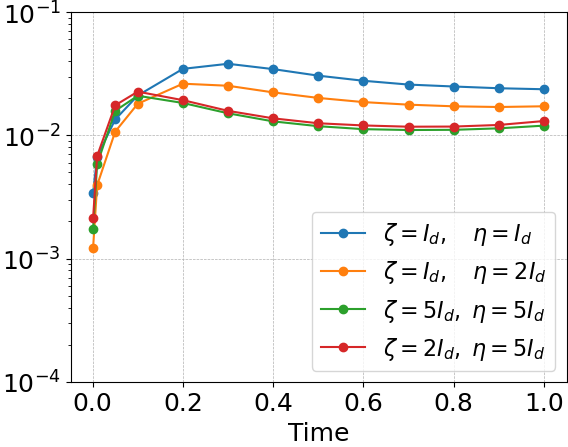}
        \subcaption{\(p(\bm{x},t)\)}
        \label{fig:Benes_high_PDF_error}
	\end{minipage}
    \caption{The relative \(L^2\) errors in the six-dimensional case. Left: \(\widehat{\text{Rel}}(p_{\text{C-KRnet}}(\cdot,t|\cdot))\); Right: \(\widehat{\text{Rel}}(\hat{p}_{\bm{\theta}}(\cdot,t))\) for various initial distributions with $M=10^4$.}
    \label{fig:Benes_high_err}
    \end{figure}

Fig.~\ref{fig:Benes_high_TPDF_error} reports the 
\(\widehat{\text{Rel}}(p_{\text{C-KRnet}}(\cdot,t|\cdot))\) where the test set \(\{\bm{x}_0,\bm{x}\}\) consists
of \(2\times 10^6\) samples. It
shows that the conditional normalizing flow achieves close agreement with the exact solution. Fig.~\ref{fig:Benes_high_PDF_error} presents \(\widehat{\text{Rel}}(\hat{p}_{\bm{\theta}}(\cdot,t))\) with a total of \(10^5\) test points drawn from \(p(\cdot,t)\) for different Beta distributions as initial conditions.
For each initial condition, the numerical solution is obtained using the equation \eqref{eqn:important_sampling}, where \(q=q_2\) is defined in the equation \eqref{eqn:mixture_proposal_t} with \(\alpha(t)=e^{-6t}\) and \(M=10^4\).
Together, these results demonstrate the effectiveness of the proposed method for moderately high-dimensional FPEs.

\subsection{SDE with nonlinear drift and constant diffusion}
 In this section, we consider the following SDE with various initial distributions over \([-1,1]^2\),
\begin{equation}
\label{eqn:ex_Nonlinear_constant_diffu}
\mathrm{d}\left(\begin{array}{c}X_1\\
X_2\end{array}\right)=\left(
\begin{array}{c}
2X_2\\
2X_1-0.8X_2-0.2X_1^3
\end{array}\right)\mathrm{d}t+\left(\begin{array}{cc}
    \sqrt{0.4} &  \\
     & \sqrt{0.8}
\end{array}\right)\mathrm{d}\bm{W}_t.
\end{equation}
Both the structure of the conditional normalizing flow and the training setting are the same as in the two-dimensional Bene\v{s} equation.
After obtaining the transition PDF, we approximate the PDF corresponding to various Beta initial distributions using the importance sampling \eqref{eqn:important_sampling} where \(q=q_2\) \eqref{eqn:mixture_proposal_t} with \(\alpha(t)=\exp(-6t)\) and \(M=10^4\).

Fig.~\ref{fig:ExampleA_err_various_beta} presents the relative \(L^2\) error \(\widehat{\text{Rel}}(\hat{p}_{\bm{\theta}}(\cdot,t))\) \eqref{eqn:relative_L2_PDF} for different initial conditions, which is evaluated using a test set consisting of \(10^4\) uniformly spaced grid points over \([-6,6]^2\). We also report the MMD \eqref{eqn:MMD} with a sample size of \(5\times 10^4\) in Fig.~\ref{fig:ExampleA_MMD_uniform}. These results demonstrate that the conditional normalizing flow yields highly accurate approximations for small  \(t\), as the base distribution is quite close to the exact solution. It also provides reliable approximations as \(t\) increases.

\begin{figure}[htbp]
	\centering
	\begin{minipage}[b]{0.35\linewidth}
		\includegraphics[height=4.4cm,width=5.2cm]{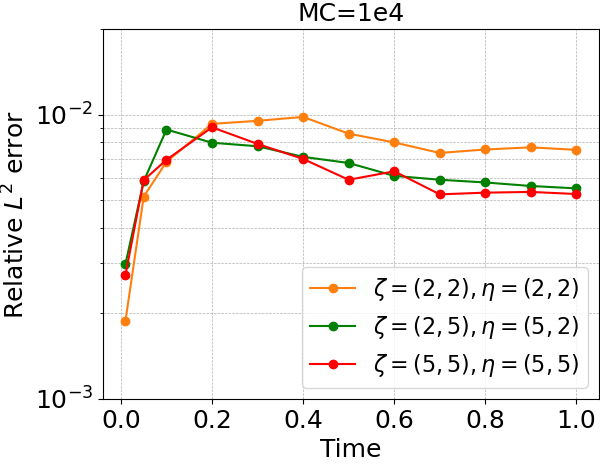}
        \subcaption{\(p(\bm{x},t)\)}
        \label{fig:ExampleA_err_various_beta}
	\end{minipage}
	\begin{minipage}[b]{0.35\linewidth}
		\includegraphics[height=4.4cm,width=5.2cm]{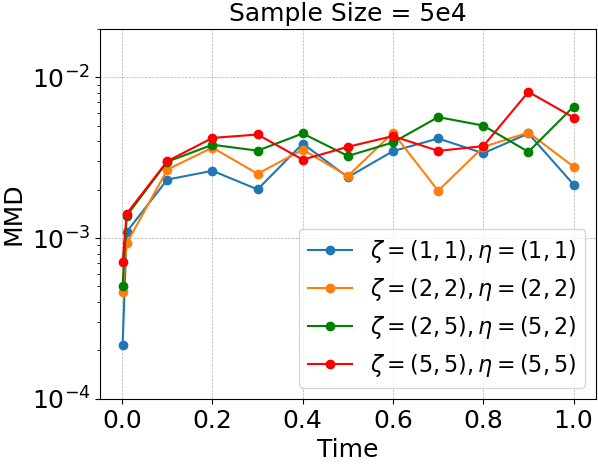}
        \subcaption{\(p(\bm{x},t)\)}
        \label{fig:ExampleA_MMD_uniform}
	\end{minipage}
    \caption{Errors for various initial distributions in the nonlinear-drift, constant-diffusion case. Left: \(\widehat{\text{Rel}}(\hat{p}_{\bm{\theta}}(\cdot,t))\). Right: MMD.}
\end{figure}

\subsection{SDE with nonlinear drift and state-dependent diffusion}
For our final example, we validate the robustness of the proposed method in approximating the PDF of the following SDE with a state-dependent diffusion,

\begin{equation*}
    \label{eqn:ex_Nonlinear_var_diffu}
\left\{\begin{split}
&\mathrm{d}\left(\begin{array}{c}X_1\\
X_2\end{array}\right)=\left(
\begin{array}{c}
X_2\\
-0.5X_1-0.3X_1^3+\sin X_2
\end{array}\right)\mathrm{d}t+\left(\begin{array}{cc}
    0.5+0.3 X_1 &  \\
     & 0.4+0.1\sin X_2
\end{array}\right)\mathrm{d}\bm{W}_t,\\
&\bm{X}_0=\bm{x}_0, \quad \bm{x}_0\in[-1,1]^2.
\end{split}\right.
\end{equation*}

The conditional normalizing flow adopts the same structure as in the two-dimensional Bene\v{s} equation. Training is conducted over 4 adaptive iterations, each consisting of 1000 epochs. The learning rate is halved every 2000 epochs. \(4\times 10^5\) training points are used, with a batch size of \(5\times 10^4\) per epoch. 

\begin{figure}[htbp]
	\centering
	\begin{minipage}[b]{0.35\linewidth}
		\includegraphics[height=4.4cm,width=5.2cm]{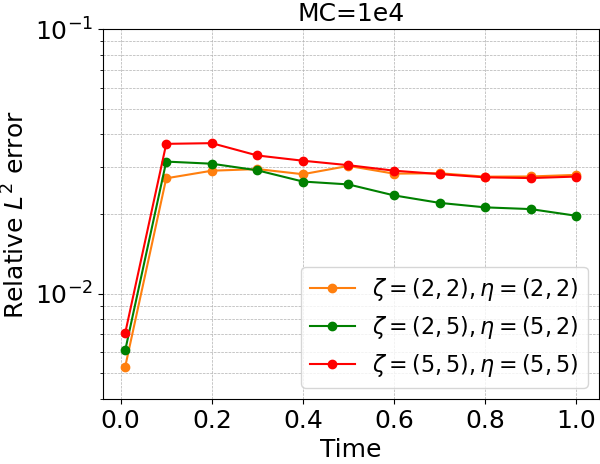}
        \subcaption{}
        \label{fig:Nonlinear-drift-var-diffu_err_various_beta}
	\end{minipage}
	\begin{minipage}[b]{0.35\linewidth}
		\includegraphics[height=4.4cm,width=5.2cm]{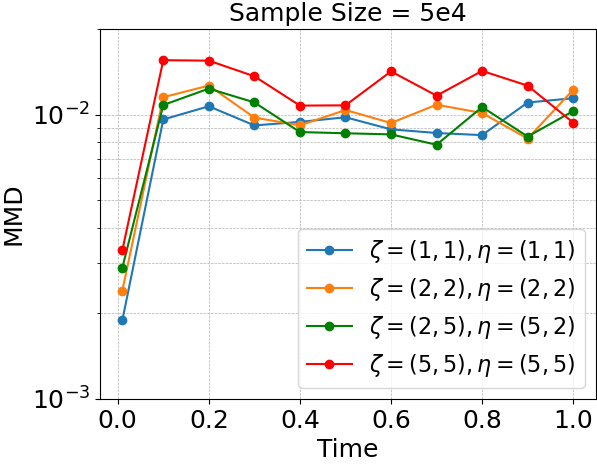}
        \subcaption{}
        \label{fig:Nonlinear-drift-var-diffu_MMD_uniform}
	\end{minipage}
    \caption{Errors for various initial distributions in the nonlinear-drift, state-dependent-diffusion case. Left: \(\widehat{\text{Rel}}(\hat{p}_{\bm{\theta}}(\cdot,t))\). Right: MMD.}
\end{figure}

Fig.~\ref{fig:Nonlinear-drift-var-diffu_err_various_beta} shows the relative \(L^2\) error \(\widehat{\text{Rel}}(\hat{p}_{\bm{\theta}}(\cdot,t))\) \eqref{eqn:relative_L2_PDF} for some initial conditions. The PDF \(\hat{p}_\theta(\cdot,t)\) is evaluated using the importance sampling estimator \eqref{eqn:important_sampling} where \(q=q_2\) with \(\alpha(t)=\exp(-6t)\) and \(10^4\) Monte Carlo samples, while the relative \(L^2\) error \(\widehat{\text{Rel}}(\hat{p}_\theta(\cdot,t))\) is computed on a test set consisting of \(10^4\) uniformly spaced grid points over \([-5,5]^2\), with the solution obtained by the ADI scheme used as the reference. 
We report the MMD \eqref{eqn:MMD} computed using \(5\times 10^4\) samples, as a function of time in Fig.~\ref{fig:Nonlinear-drift-var-diffu_MMD_uniform}.

\section{Conclusion}\label{sec:conclusion}
In this work, we proposed a transition-density-based operator learning method for solving the FPE under various initial conditions. The conditional normalizing flow is employed to approximate the transition PDF given any initial state, effectively constructing a mapping from the target stochastic process to a simple base process. The solution of the linearized SDE serves as a suitable candidate for the base stochastic process. We quantified the discrepancy between the PDFs of the linearized and original SDEs, which is characterized by differences in their drift and diffusion terms. 
By employing the solution of the linearized SDE and enforcing the identity mapping at \(t=0\), the initial condition is naturally satisfied by our conditional normalizing flow, thereby bypassing the explicit handling of the Dirac delta initial condition. In addition, we designed a time-weighted loss function to reduce the dominance of short-time residuals during training.
To evaluate \(p(\cdot, t)\) for arbitrary initial conditions over bounded domains, we used an importance sampling strategy to enhance accuracy.
A series of numerical experiments validated the effectiveness and robustness of the proposed method. The six-dimensional Bene\v{s} SDE test demonstrated that the proposed method has the potential for solving moderately high-dimensional FPEs. This research contributes to the numerical solution of complex FPEs by combining the approximation capability of the base model (the linearized SDE) with the expressive power of the conditional normalizing flow. In future work, we will further refine this framework by improving the accuracy of the base stochastic process and exploring its potential applications.
Another promising direction is to generalize the proposed method to a broader class of evolutionary partial differential equations.

\section*{Acknowledgements}
The first author was supported by the NSF of China under Grant No.~12401566. The second author was supported by the U.S. National Science Foundation under Grant No.~DMS-2513234. The last author was supported by the Science challenge project TZ2025006,  the NSF of China under grants 12288201 and 12461160275, and the Strategic Priority Research Program of the Chinese Academy of Sciences under grants XDA0480504.
The calculations have been performed using the facilities of the Scientific IT and Application Support Center of EPFL. 
\appendix

\section{Proof of Proposition~\ref{prop:p_L1}}
\label{proof_prop_p_L1}
\noindent\textbf{Proof.}
  Applying the triangle inequality and Tonelli's theorem, and using \(\int_{\Omega_0} p_0(\bm{x}_0)\mathrm{d}\bm{x}_0=1\), yields
\begin{equation}
  \begin{aligned}
\|p(\bm{x},t)-\hat{p}(\bm{x},t)\|_{L^1(\mathbb{R}^d)}&=\int_{\mathbb{R}^d}\left|\int_{\Omega_0}(p(\bm{x},t|\bm{x}_0)-\hat{p}(\bm{x},t|\bm{x}_0))p_0(\bm{x}_0)\mathrm{d}\bm{x}_0\right|\mathrm{d}\bm{x}\\
&\leq \int_{\mathbb{R}^d}\int_{\Omega_0}\left|p(\bm{x},t|\bm{x}_0)-\hat{p}(\bm{x},t|\bm{x}_0)\right|p_0(\bm{x}_0)\mathrm{d}\bm{x}_0\mathrm{d}\bm{x}\\
&=\int_{\Omega_0}\int_{\mathbb{R}^d}\left|p(\bm{x},t|\bm{x}_0)-\hat{p}(\bm{x},t|\bm{x}_0)\right|p_0(\bm{x}_0)\mathrm{d}\bm{x}\mathrm{d}\bm{x}_0\\
&\leq \text{ess}\sup_{\bm{x}_0\in\Omega_0}\|p(\bm{x},t|\bm{x}_0)-\hat{p}(\bm{x},t|\bm{x}_0)\|_{L^1(\mathbb{R}^d)}\int_{\Omega_0}p_0(\bm{x}_0)\mathrm{d}\bm{x}_0\\
&=\text{ess}\sup_{\bm{x}_0\in\Omega_0}\|p(\bm{x},t|\bm{x}_0)-\hat{p}(\bm{x},t|\bm{x}_0)\|_{L^1(\mathbb{R}^d)}.
  \end{aligned}
\end{equation}

\section{Proofs of the propositions in Section~\ref{sec:conditional_normalizing_flow}}
\subsection{Proof of Proposition~\ref{prop:dis_TPDFs}}
\label{appendix:proof_prop_dis_TPDFs} 
\noindent\textbf{Proof.}
  For each \(t>0\), the solution of the linearized SDE is a Gaussian random vector with PDF \( p_{\sigma_t}(\bm{x})=\mathcal{N}(\bm{x};\,\mathbb{E}_{\sigma_t}[\widehat{\bm{X}}_t|\widehat{\bm{X}}_0=\bm{x}_0],\bm{\Sigma}_t[\widehat{\bm{X}}_t|\widehat{\bm{X}}_0=\bm{x}_0])\) given the initial state \(\bm{x}_0\). For notation simplicity, we omit the dependence on the initial state \(\bm{x}_0\) and write \(p_{\sigma_t}(\bm{x})=\mathcal{N}(\bm{x};\,\mathbb{E}_{\sigma_t}[\widehat{\bm{X}}_t],\bm{\Sigma}_t[\widehat{\bm{X}}_t])\).
 Moreover, the linearized SDE and the original SDE share the same constant diffusion matrix, and \(\nabla \bm{f}_{\nu}\) is bounded by assumption. In particular, if \(C_{\bm{f}}=0\), then \(\bm{f}_{\sigma}(\bm{x})=\bm{f}_{\nu}(\bm{x})\) and consequently \(\sigma_t=\nu_t\), which completes the proof. It remains to consider the case \(C_{\bm{f}}>0\). It is easy to check that \((1+|\bm{x}|)^{-2}\left|G^{i j}\right|,(1+|\bm{x}|)^{-1}\left|\bm{f}_\nu\right|,(1+|\bm{x}|)^{-1}|\bm{f}_{\nu}-\bm{f}_{\sigma}| \in L^1\left(\mathbb{R}^d \times[0, t_f], \{\sigma_t\}_{0<t<t_f}\right)\) for some finite \(t_f\) where \(\|f\|_{L^1\left(\mathbb{R}^d \times[0, t_f], \{\sigma_t\}_{0<t<t_f}\right)}=\int_0^{t_f}\int_{\mathbb{R}^d}|f(\bm{x})|\sigma_t(\mathrm{d}\bm{x})\mathrm{d}t.\) Thus, according to Theorem 1.1 in \cite{bogachev2016distances}, we have
	\begin{equation}
    \label{eqn:KL_div}
		H\left(\sigma_t \mid \nu_t\right) \leq \frac{1}{2} \int_0^t \int_{\mathbb{R}^d}\left|\bm{G}^{-1 / 2} (\bm{f}_{\sigma}(\bm{x})-\bm{f}_{\nu}(\bm{x}))\right|^2 \mathrm{d} \sigma_s \mathrm{d}s.
	\end{equation}
    Combining this result with the Pinsker-Csisz\'ar-Kullback inequality \cite{gilardoni2010pinsker} yields
	\begin{equation}
    \label{eqn:TV}
		\left\|\nu_t-\sigma_t\right\|_{\mathrm{TV}}^2 \leq  \frac{1}{4}\int_0^t \int_{\mathbb{R}^d}\left|\bm{G}^{-1 / 2}\left(\bm{f}_{\sigma}(\bm{x})-\bm{f}_{\nu}(\bm{x})\right)\right|^2 \mathrm{d} \sigma_s \mathrm{d} s.
	\end{equation}
    Notice that \(\bm{G}_{\sigma}=\bm{G}_{\nu}=\bm{G}\) is a constant, non-degenerate matrix, 
	\begin{equation}
    \label{eqn:G_drfit_L2_dist}
	\left|\bm{G}^{-1/2}(\bm{f}_{\sigma}-\bm{f}_{\nu})\right|^2=(\bm{f}_{\sigma}-\bm{f}_{\nu})^T\bm{G}^{-1}(\bm{f}_{\sigma}-\bm{f}_{\nu})\leq \frac{1}{\lambda_{\text{min}}(\bm{G})}\left|\bm{f}_{\sigma}-\bm{f}_{\nu}\right|^2,
	\end{equation}
	where \(\lambda_{\text{min}}(\bm{G})\) is the smallest eigenvalue of \(\bm{G}\). 
	Meanwhile, for any \(s>0\), we have
    \begin{equation*}
    \label{eqn:drift_L2_distance}
		\begin{aligned}
		&\int_{\mathbb{R}^d}\left|\bm{f}_{\sigma}(\bm{x})-\bm{f}_{\nu}(\bm{x})\right|^2 \mathrm{d} \sigma_s 
		\leq 2\int_{\mathbb{R}^d}\left|\bm{f}_{\sigma}(\bm{x})-\bm{f}_{\sigma}(\bm{x}_0)\right|^2+\left|\bm{f}_{\nu}(\bm{x}_0)-\bm{f}_{\nu}(\bm{x})\right|^2 \mathrm{d} \sigma_s \\
		&=2\int_{\mathbb{R}^d}\left|\nabla \bm{f}_{\nu}(\bm{x}_0)(\bm{x}-\bm{x}_0)\right|^2 +\left|\int_0^1\nabla \bm{f}_{\nu}(\bm{x}_0+\tau(\bm{x}-\bm{x}_0))(\bm{x}-\bm{x}_0)\mathrm{d}\tau\right|^2\mathrm{d} \sigma_s \\
		&\leq 4C_{\bm{f}}^2 \int_{\mathbb{R}^d}\left|\bm{x}-\bm{x}_0\right|^2 \mathrm{d} \sigma_s =4C_{\bm{f}}^2\text{tr}\left(\bm{\Sigma}_s[\widehat{\bm{X}}_s]\right)+4C_{\bm{f}}^2\left|\mathbb{E}_{\sigma_s}[\widehat{\bm{X}}_s]-\bm{x}_0\right|^2,
		\end{aligned}
	\end{equation*}
    where the second inequality holds because \(\|\nabla\bm{f}_{\nu} (\bm{x})\|\leq C_{\bm{f}}\), \(\forall \bm{x}\in\mathbb{R}^d\) and the last equality holds because \(\sigma_s\) is a Gaussian measure with mean \(\mathbb{E}_{\sigma_s}[\widehat{\bm{X}}_s]\) and covariance \(\bm{\Sigma}_s[\widehat{\bm{X}}_s]\). 
    We give upper bounds for the two terms \(\text{tr}\left(\bm{\Sigma}_s[\widehat{\bm{X}}_s]\right)\) and \(|\mathbb{E}_{\sigma_s}[\widehat{\bm{X}}_s]-\bm{x}_0|^2\).
    	\begin{equation*}
    \label{eqn:linear_TPDF_cov_trace}
		\begin{aligned}
	&\text{tr}\left(\bm{\Sigma}_s[\widehat{\bm{X}}_s]\right)=\int_0^s\text{tr}(e^{\nabla \bm{f}(\bm{x}_0)(s-\tau)}\bm{g}\bm{g}^\top e^{(\nabla \bm{f}(\bm{x}_0))^\top(s-\tau)})\mathrm{d}\tau\\
	&\leq 2\text{tr}(\bm{G})\int_0^s\sigma_{\text{max}}^2\left(e^{\nabla \bm{f}(\bm{x}_0)\tau}\right)\mathrm{d}\tau\\
    &\leq 2\text{tr}(\bm{G})\int_0^s\left(e^{\sigma_{\text{max}}(\nabla \bm{f}(\bm{x}_0)\tau)}\right)^2\mathrm{d}\tau\leq 2\text{tr}(\bm{G})\int_0^se^{2C_{\bm{f}}\tau}\mathrm{d}\tau=\text{tr}(\bm{G})\frac{e^{2C_{\bm{f}}s}-1}{C_{\bm{f}}},
		\end{aligned}
	\end{equation*}
   where \(\sigma_{\text{max}}(\bm{A})\) is the spectral norm of the matrix \(\bm{A}\). 
	Meanwhile, equation \eqref{eqn:mean_var_of_linearized_SDE} gives
\begin{equation*}
\label{eqn:linear_TPDF_center_moment}
\begin{aligned}
|\mathbb{E}_{\sigma_s}[\widehat{\bm{X}}_s]-\bm{x}_0|^2= \left|\int_0^se^{\nabla\bm{f}(\bm{x}_0)\tau}\bm{f}(\bm{x}_0)d\tau\right|^2
		\leq |\bm{f}(\bm{x}_0)|^2\left(\frac{e^{C_{\bm{f}}s}-1}{C_{\bm{f}}}\right)^2.
\end{aligned}
\end{equation*}
Therefore, given \(\bm{x}_0\), \(\text{tr}\left(\bm{\Sigma}_s[\widehat{\bm{X}}_s]\right)= \mathcal{O}(s)\), \(\left|\mathbb{E}_{\sigma_s}[\widehat{\bm{X}}_s]-\bm{x}_0\right|^2=\mathcal{O}(s^2)\).
Moreover,
\begin{equation*}
		\begin{aligned}
		&H\left(\sigma_t \mid \nu_t\right) \leq \frac{1}{2\lambda_{\min}(\bm{G})} \int_0^t \int_{\mathbb{R}^d}\left|\bm{f}_{\sigma}(\bm{x})-\bm{f}_{\nu}(\bm{x})\right|^2 \mathrm{d} \sigma_s \mathrm{d} s,\\
        &\leq \frac{2C_{\bm{f}}^2}{\lambda_{\min}(\bm{G})} \int_0^t\left(\text{tr}(\bm{G})\frac{e^{2C_fs}-1}{C_{\bm{f}}}+|\bm{f}(\bm{x}_0)|^2\left(\frac{e^{C_{\bm{f}}s}-1}{C_{\bm{f}}}\right)^2\right)\mathrm{d}s\\
        & =\frac{2C_{\bm{f}}^2}{\lambda_{\min}(\bm{G})}\left( \text{tr}(\bm{G})\left(t^2+\mathcal{O}(t^3)\right)+|\bm{f}(\bm{x}_0)|^2\left(\frac{1}{3}t^3+\mathcal{O}(t^4)\right)\right)=\mathcal{O}(t^2) \text{ as } t\to 0.
		\end{aligned}
	\end{equation*}
Hence, \(\left\|\left(\nu_t-\sigma_t\right)\right\|_{\mathrm{TV}}^2\leq \frac{1}{2}H\left(\sigma_t \mid \nu_t\right)=\mathcal{O}(t^2) \text{ as } t\to 0.\)
Furthermore, if the Hessian of the drift \(\bm{f}_{\nu}\) is uniformly bounded by \(H_{\bm{f}}\), then
        \begin{align*}
            &\int_{\mathbb{R}^d}\left|\bm{f}_{\sigma}(\bm{x})-\bm{f}_{\nu}(\bm{x})\right|^2 \mathrm{d} \sigma_s \\
            =&\int_{\mathbb{R}^d}\left|\int_0^1(1-h)(\bm{x}-\bm{x}_0)^\top \text{Hess}(\bm{x}_0+h(\bm{x}-\bm{x}_0))(\bm{x}-\bm{x}_0)\mathrm{d}h\right|^2 \mathrm{d} \sigma_s \leq \frac{H_{\bm{f}}^2}{4}\int_{\mathbb{R}^d}|\bm{x}-\bm{x}_0|^4\mathrm{d}\sigma_s.
        \end{align*}
    \begin{equation*}
    \begin{aligned}
        &\int_{\mathbb{R}^d}|\bm{x}-\bm{x}_0|^4\mathrm{d}\sigma_s \leq 8\int_{\mathbb{R}^d}|\bm{x}-\mathbb{E}_{\sigma_s}[\widehat{\bm{X}}_s]|^4\mathrm{d}\sigma_s+8\int_{\mathbb{R}^d}|\mathbb{E}_{\sigma_s}[\widehat{\bm{X}}_s]-\bm{x}_0|^4\mathrm{d}\sigma_s\\
        & \leq 24\text{tr}^2\left(\bm{\Sigma}_s[\widehat{\bm{X}}_s]\right)+8\int_{\mathbb{R}^d}\left|\int_0^s e^{\nabla \bm{f}(\bm{x}_0)\tau}\bm{f}(\bm{x}_0)d\tau\right|^4\mathrm{d}\sigma_s\\
        & \leq 24\text{tr}^2(\bm{G})\frac{(e^{2C_{\bm{f}}s}-1)^2}{C_{\bm{f}}^2}+8|\bm{f}(\bm{x}_0)|^4\left(\frac{e^{C_{\bm{f}}s}-1}{C_{\bm{f}}}\right)^4
         =96 \text{tr}^2(\bm{G})s^2 + \mathcal{O}(s^3), \quad s\to 0,
    \end{aligned}
    \end{equation*}
    having used that \(\int_{\mathbb{R}^d}|\bm{x}-\mathbb{E}_{\sigma_s}[\widehat{\bm{X}}_s]|^4\mathrm{d}\sigma_s=\text{tr}^2\left(\bm{\Sigma}_s[\widehat{\bm{X}}_s]\right)+2\text{tr}\left(\bm{\Sigma}_s(\widehat{\bm{X}}_s)^2\right)\) and \(\text{tr}\left(\bm{\Sigma}_s^2[\widehat{X}_s]\right)\leq \text{tr}^2\left(\bm{\Sigma}_s[\widehat{\bm{X}}_s]\right)\).
    Hence, \(\int_0^t \int_{\mathbb{R}^d}\left|\bm{f}_{\sigma}(\bm{x})-\bm{f}_{\nu}(\bm{x})\right|^2 \mathrm{d} \sigma_s \mathrm{d} s\leq \frac{H_{\bm{f}}^2}{4}\int_0^t\int_{\mathbb{R}^d}|\bm{x}-\bm{x}_0|^4\mathrm{d}\sigma_s\mathrm{d}s=8H_{\bm{f}}^2\text{tr}^2(\bm{G})t^3+\mathcal{O}(t^4)\), as \(t\to 0\).
Therefore, \(
		H\left(\sigma_t \mid \nu_t\right) =\mathcal{O}(t^3), \; \left\|\nu_t-\sigma_t\right\|_{T V}^2=\mathcal{O}(t^3)\) as  \( t\to 0\).

\subsection{Proof of Proposition~\ref{prop:bound_expection_by_TV}}\label{proof_prop_bound_expection_by_TV}
\begin{proof}
Let \(\rho=\nu+\sigma\), then \(\nu\ll\rho\), \(\sigma\ll\rho\). Define \(p_{\nu}\coloneqq \frac{d\nu}{d\rho}\), \(p_{\sigma}\coloneqq \frac{d\sigma}{d\rho}\), thus \(\nu=p_\nu\rho\), \(\sigma=p_\sigma\rho\), and \(p_\nu+p_\sigma=1\). For \(\bm{Q}=(Q^1, \dots, Q^m)\in L^2(\mathbb{R}^d, \nu;\mathbb{R}^m)\cap L^2(\mathbb{R}^d, \sigma;\mathbb{R}^m)\), 
\begin{equation*}
	\begin{aligned}
	&\left|\int_{\mathbb{R}^d}\bm{Q}(\bm{x})\mathrm{d}\nu(\bm{x})-\int_{\mathbb{R}^d}\bm{Q}(\bm{x})\mathrm{d}\sigma(\bm{x})\right|=\left|\int_{\mathbb{R}^d}\bm{Q}(\bm{x})(p_\nu(\bm{x})-p_\sigma(\bm{x}))\mathrm{d}\rho(\bm{x})\right|\\
    \leq & \int_{\mathbb{R}^d}\left|\bm{Q}(\bm{x})(p_\nu(\bm{x})+p_\sigma(\bm{x}))^{\frac{1}{2}}\frac{p_\nu(\bm{x})-p_\sigma(\bm{x})}{(p_\nu(\bm{x})+p_\sigma(\bm{x}))^{\frac{1}{2}}}\right|\mathrm{d}\rho(\bm{x})\\
    \leq & \left(\int_{\mathbb{R}^d}|\bm{Q}(\bm{x})|^2(p_\nu(\bm{x})+p_\sigma(\bm{x}))\mathrm{d}\rho(\bm{x})\right)^{\frac{1}{2}}\left(\int_{\mathbb{R}^d}\frac{(p_\nu(\bm{x})-p_\sigma(\bm{x}))^{2}}{p_\nu(\bm{x})+p_\sigma(\bm{x})}\mathrm{d}\rho(\bm{x})\right)^{\frac{1}{2}}\\
      \leq & \left(\|\bm{Q}\|_{L^2(\mathbb{R}^d,\nu;\mathbb{R}^m)}+\|\bm{Q}\|_{L^2(\mathbb{R}^d,\sigma;\mathbb{R}^m)}\right)\sqrt{2\|\nu-\sigma\|_{\mathrm{TV}}}. \hspace{4.7cm} 
	\end{aligned}
\end{equation*} 
\qed

\end{proof}

\subsection{Proof of Proposition~\ref{prop:L_2_convergence_of_T}}
\label{proof_prop_L2_convergence_of_T}

\noindent\textbf{Proof.}
Let \(\Pi_k:\mathbb{R}^d\to \mathbb{R}^k\) denote the canonical projection onto the first \(k\) coordinates, that is, \(\Pi_k(x_1, x_2, \cdots, x_d)=(x_1, x_2, \cdots, x_k)\). 
The KR map \(T_i\) admits \(T_i(\bm{x})=(T_i^1(x_1), T_i^2(x_1, x_2),\cdots, T_i^d(x_1,\dots, x_d))\). For \(1\leq k\leq d\), denote its \(k\)-dimensional truncation by \(T_i^{(k)}(\bm{x}^{(1:k)})=\Pi_k\circ T_i(\bm{x})\) where \(\bm{x}^{(1:k)}=(x_1, \dots, x_k)\). We denote by \(\mu^{(k)}\coloneqq \Pi_{k\#}\mu\) and \(\nu_i^{(k)}\coloneqq \Pi_{k\#}\nu_i\) the \(k\)-dimensional marginal distributions of \(\mu\) and \(\nu_i\), respectively. Their corresponding PDFs are denoted by \(p_\mu^{(k)}\) and \(p_{\nu_i}^{(k)}\). By the definition of total variation distance, we have \(\|\nu_i^{(k)}-\mu^{(k)}\|_{\mathrm{TV}}\leq\|\nu_i-\mu\|_{\mathrm{TV}}\to 0.\)

We now prove \(T_i(\bm{x})\to \bm{x}\) in \(\mu\)-probability as \(i\to\infty\) by induction. 

We first show that \(T_i^1(x_1)\to x_1\) in \(\mu\)-probability as \(i\to\infty\). Let \(F_i^1\) and \(F^1\) be the cumulative distribution functions of \(\nu_i^{(1)}\) and \(\mu^{(1)}\), respectively. 
Let \(\delta_i^1=\sup_{x\in\mathbb{R}}|F_i^1(x)-F^1(x)|\). Since \(\|\nu_i^{(1)}-\mu^{(1)}\|_{\mathrm{TV}}\to 0\), we have \(\lim_{i\to\infty}\delta_i^1= 0\).
Notice that \(T_i^1(x_1)=(F^{1}_i)^{(-1)}\circ F^1(x_1)\), and \((F^{1}_i)^{(-1)}(s)=\inf\{y\in\mathbb{R}: F_i^1(y)\geq s\}\).
For any \(\epsilon>0\), if \(T_i^1(x_1)\geq x_1+\epsilon\), then \(F_i^1(x_1+\epsilon)\leq F^1(x_1)\). Hence, \(F^1(x_1+\epsilon)-F^1(x_1)\leq F^1(x_1+\epsilon)-F_i^1(x_1+\epsilon)\leq \delta_i^1\). On the other hand, if \(T_i^1(x_1)\leq x_1-\epsilon\), then \(F_i^1(x_1-\epsilon)\geq F^1(x_1)\). Hence, \(F^1(x_1)- F^1(x_1-\epsilon)\leq F_i^1(x_1-\epsilon)-F^1(x_1-\epsilon)\leq \delta_i^1\). Therefore, we have
\(\{|T_i^1(x_1)-x_1|\geq \epsilon\}\subseteq \{F^1(x_1)- F^1(x_1-\epsilon)\leq\delta_i^1\}\allowbreak\cup \{ F^1(x_1+\epsilon)-F^1(x_1)\leq\delta_i^1\}\). Let \(\Omega^{-}(\delta)=\{x_1\mid F(x_1)-F(x_1-\epsilon)\leq \delta\}\), and \(\Omega^{+}(\delta)=\{x_1\mid F(x_1+\epsilon)-F(x_1)\leq \delta\}\). Thus,
 \( \mu^{(1)}\left(|T_i^1(x_1)-x_1|\geq\epsilon\right)
 \leq \mu^{(1)}\left(\Omega^-(\delta_i^1)\right)+\mu^{(1)}\left( \Omega^{+}(\delta_i^1)\right),\)
where \(
    \mu^{(1)}\left(\Omega^-(\delta_i^1)\right)=\int 1_{\Omega^-(\delta_i^1)}(x_1)\mathrm{d}\mu^{(1)}(x_1)\), and  \(\mu^{(1)}\left(\left(\Omega^+(\delta_i^1)\right)\right)=\int 1_{\Omega^+(\delta_i^1)}(x_1)\mathrm{d}\mu^{(1)}(x_1)\). According to the dominated convergence theorem, we have
    \begin{equation*}
      \lim_{i\to \infty}\mu^{(1)}\left(\Omega^-(\delta_i^1)\right)=\int 1_{\Omega^-(0)}(x_1)\mathrm{d}\mu^{(1)}(x_1)=\mu^{(1)}(\Omega^-(0))=\sum_{j\in\mathbb{Z}}\mu^{(1)}(\Omega^{-}(0)\cap((j-1)\epsilon, j\epsilon]).
    \end{equation*}
    We note that \(x_1\in\Omega^{-}(0)\) implies \(F(x_1)-F(x_1-\epsilon)=0\).
    For any \(j\in\mathbb{Z}\), let \(\overline{x}_j=\sup\{x_1\in\Omega^-(0)\cap((j-1)\epsilon,j\epsilon])\}\), then \(\Omega^-(0)\cap((j-1)\epsilon,j\epsilon]\subset ((j-1)\epsilon,\overline{x}_j]\) and 
    \begin{equation*}
        \mu^{(1)}\left(\Omega^-(0)\right)\leq \sum_{j\in\mathbb{Z}}\mu^{(1)}\left(((j-1)\epsilon, \overline{x}_j]\right)\leq \sum_{j\in\mathbb{Z}}\mu^{(1)}((\overline{x}_j-\epsilon, \overline{x}_j])=0.
    \end{equation*}
    Thus, \(\lim_{i\to \infty}\mu^{(1)}\left(\Omega^-(\delta_i^1)\right)=0\). Similarly, we have \(\lim_{i\to \infty}\mu^{(1)}\left(\Omega^+(\delta_i^1)\right)=0\).
 Hence, \(\lim_{i\to \infty}\mu\left(\{|T_i^1(x_1)-x_1|\geq\epsilon\}\right)=0\),  and \(T_i^1(x_1)\xrightarrow{i\to\infty} x_1\) in \(\mu\)-probability.  

Suppose that for any \(1\leq k\leq d-1\), \(T_i^k(\bm{x}^{(1:k)})\xrightarrow{i\to\infty} x_k\) in \(\mu\)-probability. We next prove that \(T_i^{k+1}(\bm{x}^{(1:k+1)})\xrightarrow{i\to\infty} x_{k+1}\) in \(\mu\)-probability. To simplify the notation, we write \(\zeta_{i,T_i^{(k)}}(\mathrm{d}x_{k+1})\coloneqq\allowbreak\allowbreak \nu_i(dx_{k+1}\mid T_i^{(k)}(\bm{x}^{(1:k)}))\), \(\zeta_{i}(\mathrm{d}x_{k+1})\coloneqq\allowbreak\allowbreak \nu_i(dx_{k+1}\mid \bm{x}^{(1:k)})\),   \(\zeta_{T_i^{(k)}}(\mathrm{d}x_{k+1})\coloneqq\allowbreak\allowbreak \mu(dx_{k+1}\mid T_i^{(k)}(\bm{x}^{(1:k)}))\), \(\zeta(\mathrm{d}x_{k+1})\coloneqq\allowbreak\allowbreak \mu(\mathrm{d}x_{k+1}\mid \bm{x}^{(1:k)})\). 
We show that \begin{equation}
    \int \|\zeta_{i,T_i^{(k)}}-\zeta\|_{\mathrm{TV}}\mathrm{d}\mu^{(k)}(\bm{x}^{(1:k)})\to 0, \text{ as } i\to\infty.
    \label{eqn:k+1_th_TV}
\end{equation}
 Notice that
 \( \left\|\zeta_{i,T_i^{(k)}}-\zeta\right\|_{\mathrm{TV}}
\leq \left\|\zeta_{i,T_i^{(k)}}-\zeta_{T^{(k)}_i}\right\|_{\mathrm{TV}} + \left\|\zeta_{T^{(k)}_i}-\zeta\right\|_{\mathrm{TV}},\)
and
\begin{equation*}
\begin{aligned}
&\int \|\zeta_{i,T_i^{(k)}}-\zeta_{T_i^{(k)}}\|_{\mathrm{TV}}\mathrm{d}\mu^{(k)}(\bm{x}^{(1:k)})=\int \|\zeta_{i}-\zeta\|_{\mathrm{TV}}\mathrm{d}\nu_i^{(k)}(\bm{x}^{(1:k)})\\
=&\frac{1}{2}\int  |p_{\nu_i}(x_{k+1}|\bm{x}^{(1:k)})-p_{\mu}(x_{k+1}|\bm{x}^{(1:k)})|p^{(k)}_{\nu_i}(\bm{x}^{(1:k)})\mathrm{d} x_{k+1}\mathrm{d} \bm{x}^{(1:k)}\\
\leq & \frac{1}{2}\int \left|p_{\nu_i}^{(k+1)}-p_{\mu}^{(k+1)}\right|\mathrm{d} \bm{x}^{(1:k+1)}+\frac{1}{2}\int \left|
p^{(k)}_{\mu}-p_{\nu_i}^{(k)}\right|\mathrm{d} \bm{x}^{(1:k)}\\
= &\|\nu_i^{(k+1)}-\mu^{(k+1)}\|_{\mathrm{TV}}+\|\nu_i^{(k)}-\mu^{(k)}\|_{\mathrm{TV}}\leq 2\|\nu_i-\mu\|_{\mathrm{TV}}\to 0, \quad \text{as } i\to\infty.
\end{aligned}
\end{equation*}
Therefore, to prove \eqref{eqn:k+1_th_TV}, we only need to show \(\int \|\zeta_{T_{i}^{(k)}}-\zeta\|_{\mathrm{TV}}\mathrm{d}\mu^{(k)}\) goes to 0 as \(i\to\infty\).
Since \(\mu\) is absolutely continuous and \(\int p_{\mu}(x_{k+1}|\bm{x}^{(1:k)})\mathrm{d} x_{k+1}=1\) for \(\mu^{(k)}\)-a.e. \(\bm{x}^{(1:k)}\), \(p_{\mu}(x_{k+1}|\bm{x}^{(1:k)})\) can be regarded as an \(L^1(\mathbb{R})\) function for \(\mu^{(k)}\)-a.e. \(\bm{x}^{(1:k)}\). 
According to Lemma 1.2.31 in \cite{hytonen3analysis}, for any \(\epsilon>0\), there exists a bounded continuous \(L^1(\mathbb{R})\) function \(\phi(x_{k+1}, \bm{x}^{(1:k)})\) with  \(\|\phi\|_\infty\coloneqq \sup_{\bm{z}\in\mathbb{R}^k}\|\phi(\cdot,\bm{z})\|_{L^1(\mathbb{R})}\) such that 
\[H_\phi\coloneqq \int \|p_{\mu}(x_{k+1}|\bm{x}^{(1:k)})-\phi(x_{k+1}, \bm{x}^{(1:k)})\|_{L^1(\mathbb{R})}\mathrm{d}\mu^{(k)}(\bm{x}^{(1:k)})<\epsilon.\] 
Thus, \begin{equation}
\begin{aligned}
&\int \|\zeta_{T_{i}^{(k)}}-\zeta\|_{\mathrm{TV}}\mathrm{d}\mu^{(k)}(\bm{x}^{(1:k)})\\
=&\frac{1}{2}\int \|p_{\mu}(x_{k+1}|T_{i}^{(k)}(\bm{x}^{(1:k)}))-p_\mu(x_{k+1}|\bm{x}^{(1:k)})\|_{L^1(\mathbb{R})}\mathrm{d} \mu^{(k)}(\bm{x}^{(1:k)})\\
\leq &\frac{1}{2}\int \|p_{\mu}(x_{k+1}|T_{i}^{(k)}(\bm{x}^{(1:k)}))-\phi(x_{k+1}, T_{i}^{(k)}(\bm{x}^{(1:k)}))\|_{L^1(\mathbb{R})}\mathrm{d} \mu^{(k)}(\bm{x}^{(1:k)})\\
&+\frac{1}{2}\int \|\phi(x_{k+1}, T_{i}^{(k)}(\bm{x}^{(1:k)}))-\phi(x_{k+1}, \bm{x}^{(1:k)})\|_{L^1(\mathbb{R})}\mathrm{d} \mu^{(k)}(\bm{x}^{(1:k)})
+\frac{1}{2}H_\phi.
\end{aligned}
\label{eqn:conditional_TV}
\end{equation}
Since \({T^{(k)}_i}_\#\mu^{(k)}=\nu_i^{(k)}\) and \(\|\nu_i^{(k)}-\mu^{(k)}\|_{\mathrm{TV}}\to 0\), there exists \(i_0\in\mathbb{N}\) such that \(\|\nu_i^{(k)}-\mu^{(k)}\|_{\mathrm{TV}}< \frac{\epsilon}{1+\|\phi\|_\infty} \) for any \(i\geq i_0\). Then for \(i\geq i_0\), we have 

\begin{equation}
    \begin{aligned}
       & \int \|p_{\mu}(x_{k+1}|T_{i}^{(k)}(\bm{x}^{(1:k)}))-\phi(x_{k+1}, T_{i}^{(k)}(\bm{x}^{(1:k)}))\|_{L^1(\mathbb{R})}\mathrm{d} \mu^{(k)}(\bm{x}^{(1:k)})\\
       =&\int \|p_{\mu}(x_{k+1}|\bm{x}^{(1:k)}))-\phi(x_{k+1}, \bm{x}^{(1:k)})\|_{L^1(\mathbb{R})}\mathrm{d} \nu_i^{(k)}(\bm{x}^{(1:k)})\\
       =&H_\phi+\int \|p_{\mu}(x_{k+1}|\bm{x}^{(1:k)}))-\phi(x_{k+1}, \bm{x}^{(1:k)})\|_{L^1(\mathbb{R})}\mathrm{d} (\nu_i^{(k)}-\mu^{(k)})(\bm{x}^{(1:k)})\\
       < &\epsilon+2(1+\|\phi\|_\infty)\|\nu^{(k)}_i-\mu^{(k)}\|_{\mathrm{TV}}< 3\epsilon.
    \end{aligned}
    \label{eqn:p_phi_T_i}
\end{equation}

Meanwhile, \(T_{i}^{(k)}(\bm{x}^{(1:k)})\to \bm{x}^{(1:k)}\) in \(\mu\)-probability as \(i\to\infty\) and \(\phi(x_{k+1},\bm{x}^{(1:k)})\) is a bounded continuous function. Hence, \(\phi(x_{k+1}, T_{i}^{(k)}(\bm{x}^{(1:k)}))\to \phi(x_{k+1}, \bm{x}^{(1:k)})\) in \(\mu^{(k)}\)-probability with respect to the \(L^1\) norm as \(i\to\infty\).
Since \(\phi\) is bounded, hence integrable, according to the Vitali convergence theorem \cite{folland1999real}, we have \(\int \|\phi(\cdot, T_{i}^{(k)}(\bm{x}^{1:k}))-\phi(\cdot,\bm{x}^{(1:k)})\|_{L^1(\mathbb{R})}\mathrm{d} \mu^{(k)}(\bm{x}^{1:k})\to 0\) as \(i\to\infty\). Thus, there exists \(i_1\in\mathbb{N}\) such that 
\begin{equation}
    \int \|\phi(\cdot, T_{i}^{(k)}(\bm{x}^{1:k}))-\phi(\cdot,\bm{x}^{(1:k)})\|_{L^1(\mathbb{R})}\mathrm{d} \mu^{(k)}(\bm{x}^{1:k})<\epsilon, \text{ for any } i\geq i_1.
\label{eqn:phi_T_i}
\end{equation}
Substituting equations \eqref{eqn:p_phi_T_i}-\eqref{eqn:phi_T_i} into equation \eqref{eqn:conditional_TV} gives  
\(\int \|\zeta_{i,T_{i}^{(k)}}-\zeta\|_{\mathrm{TV}}\mathrm{d}\mu^{(k)}< \frac{5}{2}\epsilon\) as \(i>\max\{i_0, i_1\}\). Hence, \(\lim_{i\to\infty}\int \|\zeta_{i,T_{i}^{(k)}}-\zeta\|_{\mathrm{TV}}\mathrm{d}\mu^{(k)}=0\) and equation \eqref{eqn:k+1_th_TV} holds.
For \(\mu^{(k)}\)-a.e. \(\bm{x}^{(1:k)}\), \(T_i^{k+1}(\bm{x}^{(1:k)},x_{k+1})\) is a one-dimensional KR map from \(\zeta\) to \(\zeta_{i,T_i^{(k)}}\). Let \(F_{i}^{k+1}\) and \(F^{k+1}\) be the cumulative distribution functions of \(\zeta\) and \(\zeta_{i,T_i^{(k)}}\), respectively. Denote \(\delta_i^{k+1}(\bm{x}^{(1:k)})=\sup_{x_{k+1}}|F_i^{k+1}(x_{k+1})-F^{k+1}(x_{k+1})|\). Then 
\[\int\delta_i^{k+1}(\bm{x}^{(1:k)})\mathrm{d}\mu^{(k)}(\bm{x}^{(1:k)})\leq 2\int \left\|\zeta_{i,T_{i}^{(k)}}-\zeta\right\|_{\mathrm{TV}}\mathrm{d}\mu^{(k)}(\bm{x}^{(1:k)})\xrightarrow{i\to\infty} 0.\]
We want to show that \(\mu^{(k+1)}\left(|T_i^{k+1}(\bm{x}^{(1:k+1)})-x_{k+1}|\geq\epsilon\right)\xrightarrow{i\to\infty} 0\) 
or equivalently, 
\[\forall\eta>0, \, \exists i_\eta\in\mathbb{N}: \, \mu^{(k+1)}\left(|T_i^{k+1}(\bm{x}^{(1:k+1)})-x_{k+1}|\geq\epsilon\right)\leq\eta, \, \forall i\geq i_\eta.\]
Let \(R_\delta(\bm{x}^{(1:k)})=\zeta\left(F^{k+1}(x_{k+1})- F^{k+1}(x_{k+1}-\epsilon)\leq\delta\right)+\zeta\left(F^{k+1}(x_{k+1}+\epsilon)-F^{k+1}(x_{k+1})\leq \delta\right)\).
Following the same argument as for \(T_i^1\) and since \(0\leq R_\delta\leq 2\), we have  for any \(\epsilon>0\), \(\int R_\delta\mathrm{d}\mu^{(k)}\to 0\) as \(\delta\to 0\). Choose \(\delta_\eta>0\) such that \(\int R_\delta\mathrm{d}\mu^{(k)}\leq \eta/2\). Since \(\lim_{i\to\infty}\int\delta_i^{k+1}(\bm{x}^{(1:k)})\mathrm{d}\mu^{(k)}=0\), there exists \(i_\eta\) s.t. \(\int\delta_i^{k+1}(\bm{x}^{(1:k)})\mathrm{d}\mu^{(k)}\leq \delta_\eta\cdot \eta/2\), \(\forall i\geq i_\eta\). Using that 
\(\{|T_i^{k+1}(\bm{x}^{(1:k+1)})-x_{k+1}|\geq \epsilon\}\cap\{\delta_i^{k+1}\leq \delta_\eta\}\subset \{F^{k+1}( x_{k+1})- F^{k+1}(x_{k+1}-\epsilon)\leq\delta_\eta\}\allowbreak\cup \{ F^{k+1}( x_{k+1}+\epsilon)-F^{k+1}(x_{k+1})\leq\delta_\eta\}\), we have   
  \begin{align*}
&\mu^{(k+1)}\left(|T_i^{k+1}(\bm{x}^{(1:k+1)})-x_{k+1}|\geq\epsilon\right)\\
=&\mu^{(k+1)}\left(|T_i^{k+1}(\bm{x}^{(1:k+1)})-x_{k+1}|\geq\epsilon, \delta_i^{k+1}(\bm{x}^{(1:k)})\leq \delta_\eta\right)\\
&+\mu^{(k+1)}\left(|T_i^{k+1}(\bm{x}^{(1:k+1)})-x_{k+1}|\geq\epsilon, \delta_i^{k+1}(\bm{x}^{(1:k)})> \delta_\eta\right)\\
\leq& \int\zeta(|T_i^{k+1}(\bm{x}^{(1:k+1)})-x_{k+1}|\geq\epsilon)1_{\delta_i^{k+1}(\bm{x}^{(1:k)})\leq \delta_\eta}\mathrm{d}\mu^{(k)}+
\mu^{(k+1)}(\delta_i^{k+1}>\delta_\eta)\\
\leq &\int R_{\delta_\eta}(\bm{x}^{(1:k)})\mathrm{d}\mu^{(k)}(\bm{x}^{(1:k)})+\mu^{k+1}(\delta_i^{k+1}>\delta_\eta)
\leq \frac{\eta}{2}+\frac{1}{\delta_\eta}\int\delta_i^{k+1}(\bm{x}^{(1:k)})\mathrm{d}\mu^{(k)}\leq \eta, \, \forall i\geq i_\eta,
  \end{align*}
which shows that \(T^{k+1}_i(\bm{x}^{(1:k+1)})\to x_{k+1}\) in \(\mu^{(k+1)}\)-probability.

We now prove that \(\lim_{i\to\infty}\int_{\mathbb{R}^d}|T_i(\bm{x})-\bm{x}|^2\mathrm{d}\mu(\bm{x})= 0\). Since \(T_{i_\#}\mu=\nu_i\), 
\begin{equation*}
\begin{aligned}
\int_{\mathbb{R}^d}|T_i(\bm{x})|^2\mathrm{d}\mu(\bm{x})&=\int_{\mathbb{R}^d}|\bm{x}|^2\mathrm{d}\nu_i(\bm{x})\to \int_{\mathbb{R}^d}|\bm{x}|^2\mathrm{d}\mu(\bm{x}), \quad \text{as } i\to\infty.
\end{aligned}
\end{equation*}
That is to say, \(\lim_{i\to\infty}\|T_i\|_{L^2(\mathbb{R}^d,\mu)}= \|id\|_{L^2(\mathbb{R}^d,\mu)}\). Meanwhile, \(2|T_i(\bm{x})\cdot \bm{x}|\leq |T_i(\bm{x})|^2+|\bm{x}|^2,\)
and \(\{T_i\}_{i\geq 1}\) are uniformly integrable in \(L^1(\mu)\). Thus, according to the Vitali convergence theorem \cite{folland1999real}, we have 
\(
\lim_{i\to\infty}\int_{\mathbb{R}^d}T_i(\bm{x})\cdot \bm{x}\mathrm{d}\mu(\bm{x})= \int_{\mathbb{R}^d}|\bm{x}|^2\mathrm{d}\mu(\bm{x})
\). Hence,
\begin{equation*}
  \begin{aligned}
\|T_i-id\|^2_{L^2(\mathbb{R}^d,\mu)}=&\|T_i\|^2_{L^2(\mathbb{R}^d,\mu)}+\|id\|^2_{L^2(\mathbb{R}^d,\mu)}-2\int_{\mathbb{R}^d}T_i(\bm{x})\cdot \bm{x}\mathrm{d}\mu(\bm{x})\xrightarrow{i\to\infty} 0. 
  \end{aligned}
\end{equation*}
 Therefore, \(\|T_i-id\|^2_{L^2(\mu)}\to 0\) as \(i\to\infty\). 

\section{Proof of Proposition~\ref{prop:residual_approx}}
\label{appendix:proof_prop_residual_approx}

\noindent\textbf{Proof.}
    For any \(\tilde{p}(\bm{x},t|\bm{x}_0)=\mathcal{N}(\bm{x};\bm{m}(t), \bm{\Sigma}(t)), \quad \bm{\Sigma}(t)\sim t\bm{G}, \) we have
    
   \noindent\(\frac{\partial \tilde{p}}{\partial t}=\tilde{p}\left(-\frac{1}{2}\text{tr}\left(\bm{\Sigma}^{-1}\frac{\mathrm{d}\bm{\Sigma}}{\mathrm{d}t}\right)+\frac{\mathrm{d}\bm{m}^\top}{\mathrm{d}t}\bm{\Sigma}^{-1}(\bm{x}-\bm{m})+\frac{1}{2}(\bm{x}-\bm{m})^\top\bm{\Sigma}^{-1}\frac{\mathrm{d}\bm{\Sigma}}{\mathrm{d}t}\bm{\Sigma}^{-1}(\bm{x}-\bm{m})\right),\)

    \noindent\(\nabla \tilde{p}=\tilde{p}(-\bm{\Sigma}^{-1}(\bm{x}-\bm{m})),\quad  \nabla\cdot\nabla\cdot(\bm{g}\bm{g}^\top \tilde{p})=\tilde{p}(\bm{g}^\top\bm{\Sigma}^{-1}(\bm{x}-\bm{m}))^2-\tilde{p}\bm{g}^\top\bm{\Sigma}^{-1}\bm{g}.\)
    
Hence, \(
    r(\bm{x},\bm{x}_0,t;\tilde{p})=\tilde{p}\tilde{r}(\bm{x},\bm{x}_0,t;\tilde{p}),
\)
\begin{equation*}
\begin{aligned}
    \tilde{r}(\bm{x},\bm{x}_0,t;\tilde{p})=&-\frac{1}{2}\text{tr}\left(\bm{\Sigma}^{-1}\frac{\mathrm{d}\bm{\Sigma}}{\mathrm{d}t}\right)+\frac{\mathrm{d}\bm{m}^\top}{\mathrm{d}t}\bm{\Sigma}^{-1}(\bm{x}-\bm{m})+\frac{1}{2}(\bm{x}-\bm{m})^\top\bm{\Sigma}^{-1}\frac{\mathrm{d}\bm{\Sigma}}{\mathrm{d}t}\bm{\Sigma}^{-1}(\bm{x}-\bm{m})\\
    & + \nabla\cdot \bm{f} - \bm{f}^\top\bm{\Sigma}^{-1}(\bm{x}-\bm{m})-\frac{1}{2}(\bm{g}^\top\bm{\Sigma}^{-1}(\bm{x}-\bm{m}))^2+\frac{1}{2}\bm{g}^\top\bm{\Sigma}^{-1}\bm{g}.
\end{aligned}
\end{equation*}
Thus, \(|\tilde{r}(\bm{x},\bm{x}_0,t;\tilde{p})| \leq \tilde{C}_0+\tilde{C}_1t^{-1}+\tilde{C}_2|\bm{x}-\bm{m}|t^{-1}+\tilde{C}_3|\bm{x}-\bm{m}|^2t^{-1}+\tilde{C}_4|\bm{x}-\bm{m}|^2t^{-2}.\)
\begin{align*}
    |\tilde{r}(\bm{x},\bm{x}_0,t;\tilde{p})|^2\leq & (C_0+C_1t^{-1}+C_2t^{-2})+(C_3t^{-1}+C_4t^{-2})|\bm{x}-\bm{m}|\\
    &+(C_5t^{-1}+C_6t^{-2}+C_7t^{-3})|\bm{x}-\bm{m}|^2+(C_8t^{-2}+C_9t^{-3})|\bm{x}-\bm{m}|^3\\
    &+(C_{10}t^{-2}+C_{11}t^{-3}+C_{12}t^{-4})|\bm{x}-\bm{m}|^4.
\end{align*}
Notice that \(\left(\mathbb{E}_{\bm{X}\sim\mathcal{N}\left(\bm{x};\bm{m},\bm{\Sigma}\right)}\left[|\bm{X}-\bm{m}|\right]\right)^2\leq \mathbb{E}_{\bm{X}\sim\mathcal{N}\left(\bm{x};\bm{m},\bm{\Sigma}\right)}\left[|\bm{X}-\bm{m}|^2\right]\).\\
\(\mathbb{E}_{\bm{X}\sim\mathcal{N}\left(\bm{x};\bm{m},\bm{\Sigma}\right)}\left[|\bm{X}-\bm{m}|^2\right]=\text{tr}(\bm{\Sigma})\), \(\mathbb{E}_{\bm{X}\sim\mathcal{N}\left(\bm{x};\bm{m},\bm{\Sigma}\right)}\left[|\bm{X}-\bm{m}|^4\right]=\text{tr}^2(\bm{\Sigma})+2\text{tr}(\bm{\Sigma}^2)\), \(\mathbb{E}_{\bm{X}\sim\mathcal{N}\left(\bm{x};\bm{m},\bm{\Sigma}\right)}\left[|\bm{X}-\bm{m}|^6\right]=\text{tr}^3(\bm{\Sigma})+6\text{tr}(\bm{\Sigma})\text{tr}(\bm{\Sigma}^2)+8\text{tr}(\bm{\Sigma}^3)\).
Thus, \vspace{-3pt}
\begin{align*}
    &\int_{\mathbb{R}^d}|r(\bm{x},\bm{x}_0, t;\tilde{p})|^2\mathrm{d}\bm{x}
    =(4\pi)^{-\frac{d}{2}}(\det(\bm{\Sigma}))^{-\frac{1}{2}}\mathbb{E}_{\bm{X}\sim\mathcal{N}\left(\bm{x};\bm{m},\frac{\bm{\Sigma}}{2}\right)}\left[|\tilde{r}(\bm{x},\bm{x}_0, t;\tilde{p})|^2\right]\\
    \leq&(4\pi)^{-\frac{d}{2}}(\det(\bm{\Sigma}))^{-\frac{1}{2}}\left(C_0+C_1t^{-1}+C_2t^{-2}+(C_3t^{-1}+C_4t^{-2})\sqrt{\frac{\text{tr}(\bm{\Sigma})}{2}}\right.\\
    &\left.+\frac{1}{2}(C_5t^{-1}+C_6t^{-2}+C_7t^{-3})\text{tr}(\bm{\Sigma})+(C_8t^{-2}+C_9t^{-3})\sqrt{\frac{1}{8}\text{tr}^3(\bm{\Sigma})+\frac{3}{4}\text{tr}(\bm{\Sigma})\text{tr}(\bm{\Sigma}^2)+\text{tr}(\bm{\Sigma}^3)}\right.\\
    &\left.+(C_{10}t^{-2}+C_{11}t^{-3}+C_{12}t^{-4})\left(\frac{1}{4}\text{tr}^2(\bm{\Sigma})+\frac{1}{2}\text{tr}(\bm{\Sigma}^2)\right)\right)=\mathcal{O}(t^{-d/2-2}) \text{ as } t\to0.
\end{align*}
Similarly, \vspace{-3pt}
\begin{align*}
     &\int_{\mathbb{R}^d}|r(\bm{x},\bm{x}_0, t;\tilde{p})|^2\tilde{p}\mathrm{d}\bm{x}
    =(2\sqrt{3}\pi)^{-d}(\det(\bm{\Sigma}))^{-1}\mathbb{E}_{\bm{X}\sim\mathcal{N}\left(\bm{x};\bm{m},\frac{\bm{\Sigma}}{3}\right)}\left[|\tilde{r}(\bm{x},\bm{x}_0, t;\tilde{p})|^2\right]\\
    \leq &(2\sqrt{3}\pi)^{-d}(\det(\bm{\Sigma}))^{-1}\left(C_0+C_1t^{-1}+C_2t^{-2}+(C_3t^{-1}+C_4t^{-2})\sqrt{\frac{\text{tr}(\bm{\Sigma})}{3}}\right.\\
    &\left.+\frac{1}{3}(C_5t^{-1}+C_6t^{-2}+C_7t^{-3})\text{tr}(\bm{\Sigma})+(C_8t^{-2}+C_9t^{-3})\sqrt{\frac{1}{27}\text{tr}^3(\bm{\Sigma})+\frac{2}{9}\text{tr}(\bm{\Sigma})\text{tr}(\bm{\Sigma}^2)+\frac{8}{27}\text{tr}(\bm{\Sigma}^3)}\right.\\
    &\left.+(C_{10}t^{-2}+C_{11}t^{-3}+C_{12}t^{-4})\left(\frac{1}{9}\text{tr}^2(\bm{\Sigma})+\frac{2}{9}\text{tr}(\bm{\Sigma}^2)\right)\right)=\mathcal{O}(t^{-(d+2)}) \text{ as } t\to0.
\end{align*}

\section{Numerical details of single-FPE solver}
\label{appendix:numerical_details_of_single_solver}
Let \(\bm{X}_t\) satisfy the non-degenerate SDE \eqref{sde}, with initial density \(p_0\). Its PDF \(p(\bm{x},t)\) then satisfies the following FPE:
\begin{equation}
    \label{eqn:single_tFK}
    \left\{
    \begin{aligned}
&\frac{\partial p(\bm{x},t)}{\partial t}=\mathcal{L}^*_{\bm{f},\bm{g}}p(\bm{x},t),&\bm{x}\in\mathbb{R}^d, t>0,\\
&p(\bm{x},0)=p_0(\bm{x}), &\bm{x}\in\mathbb{R}^d.
\end{aligned}
\right.
\end{equation}
We employ a temporal normalizing flow \(p_{\bm{\theta}}(\bm{x},t)\) to approximate the unknown PDF \(p(\bm{x},t)\). The temporal normalizing flow is composed of ActNorm layers and affine coupling layers conditioned on time \(t\) \cite{feng2022solving}. Fig.~\ref{fig:Single_solver_flowchart} presents the overall structure of the temporal normalizing flow.

\begin{figure}[htbp]
	\centering
		\includegraphics[scale=0.15]{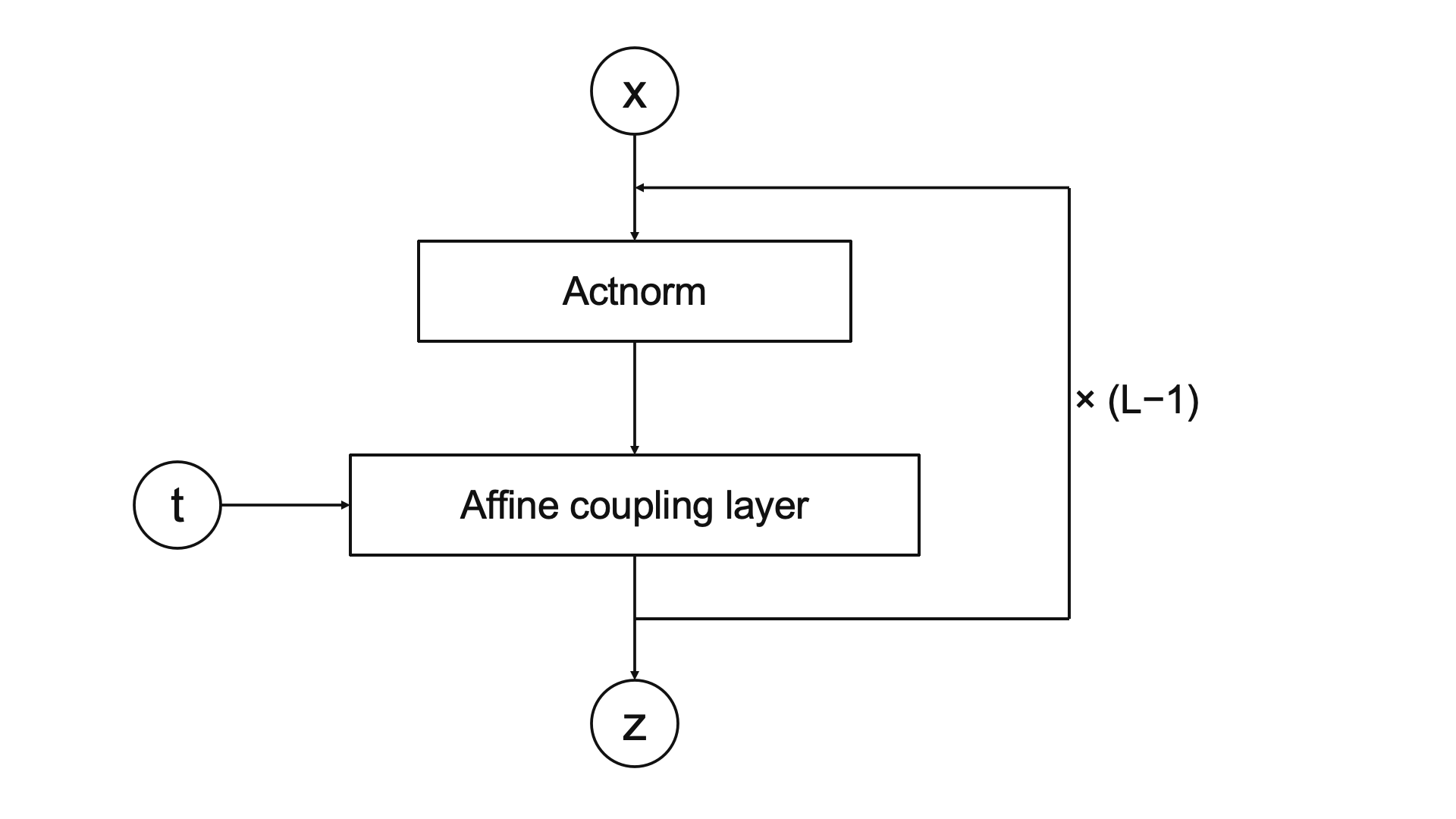}
    \caption{The structure of the temporal normalizing flow.}  
    \label{fig:Single_solver_flowchart}
\end{figure}

\subsection{The structure of the temporal normalizing flow for a single Fokker-Planck equation}
\subsubsection{ActNorm layer}
We apply an affine transformation to the input in the ActNorm layer, i.e.,
\begin{equation}
    \bm{y} = \exp(\bm{a})\odot(\bm{x}+\bm{b}),
\end{equation}
where \(\bm{a}\) and \(\bm{b}\) are trainable parameters. We initialize \(\bm{a}\) and \(\bm{b}\) with the mean and variance of a mini-batch of data. After initialization, both \(\bm{a}\) and \(\bm{b}\) are optimized during training. The transformation has an exact inverse,
\begin{equation}
    \bm{x}=\bm{y}\odot\exp(-\bm{a})-\bm{b},
\end{equation}
and a log absolute Jacobian determinant of \(\sum_{j}a_j\).

\subsubsection{Affine coupling layer}
The affine coupling layer is similar to the conditional affine coupling layer used in the operator-learning setting, but it does not depend on the initial state \(\bm{x}_0\). Let \(\bm{x}=(\bm{x}_1,\bm{x}_2)\in\mathbb{R}^d\), where \(\bm{x}_1\in\mathbb{R}^m\) and \(\bm{x}_2\in\mathbb{R}^{d-m}\). The transformation applied to \(\bm{x}\) is
\begin{equation}
\label{affine_layer_t_single}
\left\{\begin{aligned}
& \bm{y}_{1}=\bm{x}_{1},\\
& \bm{y}_{2}=\bm{x}_{2}\odot \big(\bm{1}_{m-m_1}+\beta\tanh (\bm{s})\big) + e^{\bm{\zeta}}\odot\tanh(\bm{q}),
\end{aligned}\right.
\end{equation}
where \(\beta\in(0,1)\) is user specified, \(\bm{\zeta}\) is trainable and \(\bm{s}:\mathbb{R}^{m+1}\to\mathbb{R}^{d-m}\) and \(\bm{q}:\mathbb{R}^{m+1}\to\mathbb{R}^{d-m}\) are given by a neural network taking as input \(\bm{x}_1\) and \(t\), i.e.,
\begin{equation}
    (\bm{s},\bm{q})=\text{NN}(\bm{x}_1,t).
\end{equation}
As in Section~\ref{sec:conditional_affine_coupling_layer}, we apply a random Fourier feature transformation before the fully connected layers.
\begin{equation}
\label{eqn:NN_structure_for_single_NF}
\begin{aligned}
    &\bm{h}_0=\left[\bm{x}_{1},t\right]^\top,\quad
    \bm{h}_1=\left[
        \sin\left(\frac{1}{e^{\gamma}}\bm{F}\bm{h}_0+\bm{b}_0\right),
        \cos\left(\frac{1}{e^{\gamma}}\bm{F}\bm{h}_0+\bm{b}_0\right),
        \quad \bm{h}_0\right]^\top,\\
    &\bm{h}_j = \text{SiLU}(\bm{W}_{j-1}\bm{h}_{j-1}+\bm{b}_{j-1}), \text{  for } j=2,\dots, M,\quad
    \left(
        \bm{s},
        \bm{q}
    \right)=\bm{W}_M\bm{h}_M+\bm{b}_M,
\end{aligned}
\end{equation}
where \(\bm{F}\in\mathbb{R}^{r_h/2\times\text{dim}(h_0)}\), \(\bm{b}_0\in\mathbb{R}^{r_h/2}\) are sampled from a standard normal distribution and a uniform distribution over \([0,2\pi]^{r_h/2}\), respectively, and \(r_h\) denotes the dimension of the random features. Both \(\bm{F}\) and \(\bm{b}_0\) are fixed after initialization. \(\gamma, \{\bm{W}_{j}, \bm{b}_j\}_{j=1}^{M}\) are trainable parameters, and the sigmoid linear unit (SiLU) function is used as the activation function.

\subsection{Methodology for a single-FPE solver}
The loss function combines the PDE residual with a penalty for the mismatch between the predicted and prescribed initial distributions:
\begin{equation*}
    \widehat{\mathcal{L}}_{\text{single}}(\bm{\theta}; S, S_0)=\frac{1}{N_S}\sum_{i=1}^{N_S}|r(\bm{x}^i,t^i;p_{\bm{\theta}})|^2+10\frac{1}{N_{S_0}}\sum_{i=1}^{N_{S_0}}|p(\bm{x}^i,0)-p_{\bm{\theta}}(\bm{x}^i,0)|^2.
\end{equation*}
Initial training points are drawn uniformly over \((-5,5)^2\). We update the training points every \(300\) epochs by sampling from the current temporal normalizing flow \(p_{\bm{\theta}}(\cdot, t)\) for some time \(t\).

We use a temporal normalizing flow with eight affine coupling layers and eight ActNorm layers to approximate the solution of the FPE for a given initial condition. The parameters of each coupling layer are generated by a neural network with \(32\) random Fourier features and two hidden layers of \(32\) neurons each, yielding a total parameter count comparable to that of the operator-based model.
The number of training points to approximate the PDE residual is 200000 with a batch size of \(10000\). We use \(5000\) points per epoch to approximate the initial condition violation. We use the same training procedure as for the operator model.

\bibliographystyle{unsrtnat}

\bibliography{references}

\end{document}